\documentclass{article}

\usepackage{amsmath,amsthm,amssymb}
\usepackage{graphicx}
\usepackage{booktabs}
\usepackage{algorithm}
\usepackage{algpseudocode}
\usepackage{listings}
\usepackage{caption}
\usepackage{mathtools}
\usepackage[nameinlink]{cleveref}
\usepackage{epstopdf} 
\usepackage{adjustbox}
\usepackage{tikz}
\usetikzlibrary{decorations.pathreplacing, matrix, shapes.geometric, arrows}
\usepackage{wrapfig}
\usepackage{xspace}
\usepackage{thmtools}
\usepackage{url}
\usepackage[sort]{natbib}

\definecolor{myred}{HTML}{FF435D}
\definecolor{mymagenta}{HTML}{D443FF}
\definecolor{myblue}{HTML}{7143FF}
\definecolor{mycyan}{HTML}{43CCFF}

\tikzstyle{startstop} = [rectangle, rounded corners, minimum width=2.5cm, minimum height=0.8cm, text centered, draw=black, fill=myred!50]
\tikzstyle{process} = [rectangle, minimum width=2.5cm, minimum height=0.8cm, text centered, draw=black, fill=myblue!50]
\tikzstyle{decision} = [diamond, minimum width=2.5cm, minimum height=0.8cm, text centered, draw=black, fill=mycyan!50]
\tikzstyle{arrow} = [thick,->,>=stealth]

\newtheorem{theorem}{Theorem}[section]
\newtheorem{lemma}[theorem]{Lemma}

\begin{document}

\title{Diversity-Preserving Exploitation of Crossover}

\author{Johannes Lengler \\
\small ETH Z\"{u}rich \\ \small Department of Computer Science
\and
Tom Offermann \\
\small ETH Z\"{u}rich\\ \small Department of Computer Science
}

\date{} 

\newcommand{\lb}{\left(}
\newcommand{\rb}{\right)}
\newcommand{\lbr}{\left[}
\newcommand{\rbr}{\right]}
\newcommand{\lc}{\left\{}
\newcommand{\rc}{\right\}}
\newcommand{\eps}{\varepsilon}
\newcommand{\cond}{\mid}
\newcommand{\bin}{\text{Bin}}
\newcommand{\geo}{\mathcal G}
\newcommand{\E}[1]{\mathbb{E} \lbr #1 \rbr}
\newcommand{\EE}{\mathbb{E}}
\newcommand{\RR}{\mathbb{R}}
\newcommand{\calB}{\mathcal{B}}
\newcommand{\calE}{\mathcal{E}}
\newcommand{\calL}{\mathcal{L}}
\newcommand{\Eh}[1]{\hat \mathbb{E} \lbr #1 \rbr}
\newcommand{\dega}{\textsc{DEGA}\xspace}
\newcommand{\OM}{\textsc{OM}}
\newcommand{\LO}{\textsc{LO}}
\newcommand{\NO}{\textsc{NO}}
\newcommand{\onemax}{\textsc{OneMax}\xspace}
\newcommand{\leadingones}{\textsc{Leading\-Ones}\xspace}
\newcommand{\jump}{\textsc{Jump}\xspace}
\newcommand{\ooea}{\ensuremath{(1+1)\textsc{-EA}}\xspace}
\newcommand{\tpoga}{\ensuremath{(2+1)\textsc{-GA}}\xspace}
\newcommand{\todega}{\ensuremath{(2+1)\textsc{-DEGA}}\xspace}
\newcommand{\ollga}{(\ensuremath{1+(\lambda,\lambda))}\text{-GA}\xspace}
\newcommand{\UMDA}{\textsc{UMDA}\xspace}
\newcommand{\TODO}[1]{\textcolor[rgb]{0.9,0.2,0.2}{TODO: #1}}
\newcommand{\bound}[1]{\textcolor[rgb]{0.2,0.2,0.9}{#1}} 
\newcommand{\question}[1]{\textcolor[rgb]{0.2,0.7,0.2}{Q: #1}}

\newcommand{\jl}[1]{\textbf{\textcolor{purple}{JL: #1}}}
\newcommand{\tof}[1]{\textbf{\textcolor{cyan}{TOF: #1}}}

\theoremstyle{remark}%
\newtheorem{remark}{Remark}%

\maketitle

\begin{abstract}
Crossover is a powerful mechanism for generating new solutions from a given population of solutions. Crossover comes with a discrepancy in itself: on the one hand, crossover usually works best if there is enough diversity in the population; on the other hand, exploiting the benefits of crossover reduces diversity. This antagonism often makes crossover reduce its own effectiveness. 

We introduce a new paradigm for utilizing crossover that reduces this antagonism, which we call \emph{diversity-preserving exploitation of crossover (DiPEC)}. The resulting \emph{Diversity Exploitation Genetic Algorithm (\dega)} is able to still exploit the benefits of crossover, but preserves a much higher diversity than conventional approaches. 

We demonstrate the benefits by proving that the \todega finds the optimum of \leadingones with $O(n^{5/3}\log^{2/3} n)$ fitness evaluations. This is remarkable since standard genetic algorithms need $\Theta(n^2)$ evaluations, and among genetic algorithms only some artificial and specifically tailored algorithms were known to break this runtime barrier. We confirm the theoretical results by simulations. Finally, we show that the approach is not overfitted to \leadingones by testing it empirically on other benchmarks and showing that it is also competitive in other settings. We believe that our findings justify further systematic investigations of the DiPEC paradigm.
\end{abstract}


\section{Introduction}\label{sec:intro}
Crossover, the idea of recombining two or more solutions into a new one, is a key ingredient of genetic algorithms~\cite{vcrepinvsek2013exploration,whitley2019next}. Crossover often works best when there is enough diversity in the population~\cite{dang2017escaping,sudholt2017crossover,opris2024tight,doerr2024runtime,cerf2024population}, see also the survey in~\cite{sudholt2020benefits}. In this case, crossover may serve as an exploitation mechanism which is able to find solutions that are fitter than the current population~\cite{vcrepinvsek2013exploration}. However, as usual for exploitation mechanism this comes at a cost for population diversity. In this paper we suggest the following new paradigm for utilizing crossover, which maintains some of the benefits of crossover without destroying so much diversity. \smallskip 

\noindent\textbf{The new paradigm: diversity-preserving exploitation of cross\-over (DiPEC).} For illustration, consider optimization of pseudo-Boolean functions $f:\{0,1\}^n \to \mathbb{R}$ with uniform crossover. For two parents $x^1$ and $x^2$, uniform crossover creates an offspring $y$ by selecting each bit randomly from either $x^1$ or $x^2$. Consider the case that $y$ replaces $x^1$ or $x^2$ in the population. Then the Hamming distance of the new pair $\{x^1,y\}$ or $\{x^2,y\}$ is in expectation only half as large as the Hamming distance of the previous pair $\{x^1,x^2\}$. Hence, diversity of this pair, measured by its Hamming distance, is roughly cut in half, which constitutes a massive loss of diversity. In this paper we will focus on population size $2$, where this is the normal case, and the diversity of the whole population is cut in half. However, the same principle also applies to larger population sizes.

In order to explain the new DiPEC paradigm, let us consider the example above and rephrase it. For concreteness, assume that the result of crossover $y$ is fitter than $x^1$, and let $m := x^1 \oplus y$ be the bit-wise xor of $x^1$ and $y$. A way to interpret the situation is that $m$ is a \emph{fitness-improving mask} for $x^1$: if we start from $x^1$ and \emph{apply the mask $m$}, i.e., flip all bits in $x^1$ where the mask has a one-bit, then this yields the search point $x^1\oplus m = y$ of higher fitness. So a re-interpretation of the situation is that we have found a fitness-improving mask for $x^1$, where the one-bits in the mask encode bit flips to $x^1$. The problem of this mask is that applying it destroys diversity by decreasing the Hamming distance from $x^1$ to $x^2$, roughly by a factor of $2$.

The key idea is that in such a situation, the mask $m$ can often be replaced by a smaller mask, meaning a mask that has fewer one-bits. This is based on the assumption that not all bit flips in $m$ are equally important. Often, a small number of \emph{critical bit flips} in $m$ is responsible for all or most of the fitness improvement, while many other bit flips may contribute little. The key insight is that in such a case, it may be possible to find those critical bit flips efficiently by \emph{subsampling} the mask $m$. I.e., we randomly compute another mask $m'$ where we independently keep each one-bit with some probability $1/\lambda$, while everything else is a zero-bit. Then we apply the mask $m'$ to $x^1$, yielding an offspring $y' := x^1\oplus m'$. Equivalently, $y'$ can be obtained as a \emph{biased crossover}\footnote{This is not related to the concept of unbiased operators introduced by Lehre and Witt~\cite{lehre2012black}. The biased crossover used here is still an \emph{unbiased} operator in their sense.} between $x^1$ and $y$, where we take each bit from $y$ with probability $1/\lambda$, and from $x^1$ with probability $1-1/\lambda$. Then $y'$ has a decent chance of $1/\lambda$ to contain at least the most critical bit flip from the mask $m$. By repeating the biased crossover $\lambda$ times, we have a good (constant) chance of generating at least one biased crossover that contains the most critical bit flip from the mask $m$, and hopefully the resulting offspring $y'$ is still strictly fitter than $x^1$.

Crucially, the Hamming distance $H(x^1,y')$ is much smaller than the Hamming distance $H(x^1,y)$. Hence, if we replace $x^1$ by $y'$ then we destroy much less diversity than if we replace $x^1$ by $y$. Quantitatively, while we lose in expectation a factor $1/2$ with $y$, i.e., $\EE[H(x^2,y)] = \tfrac12 H(x^1,x^2)$, we only lose a factor of $1-\tfrac{1}{2\lambda}$ with $y'$, i.e., $\EE[H(x^2,y')] = (1-\tfrac1{2\lambda}) H(x^1,x^2)$. This is a massive difference for preserving diversity, and we will show in this paper that it is often a worthwhile trade-off. Hence we suggest not to consider $y$ for inclusion into the population, but the fittest biased crossover offspring $y'$ out of $\lambda$ trials instead. The trade-off in a good (and hopefully typical) situation looks as follows.
 \begin{itemize}
     \item The fitness of $y'$ is almost as large as the fitness of $y$.
     \item Including $y'$ destroys very little diversity, whereas including $y$ would massively reduce diversity.
     \item This comes at the cost of $1+\lambda$ function evaluations (of $y$ and of $\lambda$ candidates for $y'$) instead of just $1$. The cost can be reduced further, as we will discuss later.
 \end{itemize}
As we will discuss later, the idea of using biased crossover to extract the most beneficial bit has been used before in the \ollga~\cite{doerr2015black}, but in a very different situation and with different aim.\smallskip

\noindent\textbf{Our contribution.}
Apart from introducing the DiPEC paradigm, in this paper we make some first steps of exploring it. Even though the DiPEC paradigm can be applied to arbitrary population sizes, and to arbitrary pairs of search points of different fitness, we focus here on population size $\mu=2$ and on $(2+1)$ algorithms. This already gives us strong effects on \leadingones and allows us to give rigorous proofs without too much technical overhead, but is certainly only a first step. We call the resulting algorithm the $(2+1)$ Diversity Exploitation Genetic Algorithm or $(2+1)$ DEGA, where we often omit the $(2+1)$ for brevity. As our main result, we analyze the $(2+1)$ \dega on the \leadingones benchmark, see Section~\ref{sec:prelim} for a definition. This is a promising benchmark for the \dega since progress is often hard, and typically happens in small steps. This means that for every improvement there is typically a single bit flip which is responsible for a substantial part of the improvement. We show by runtime analysis that the \dega finds the optimum with $\EE[T] = O(\lambda n + n^2\log n/\sqrt{\lambda})$ function evaluations.\footnote{We use $\log n$ for the natural logarithm with base $e$.} For a wide range of $\lambda$, this gives a subquadratic runtime bound, in particular $\EE[T] = O(n^{5/3}(\log n)^{2/3})$ for $\lambda= (n\log n)^{2/3}$. This is remarkable because $\Theta(n^2)$ seems a strong barrier for other genetic algorithms on \leadingones, except for some artificial tailored algorithms. We will discuss those in more detail below, and also explain why we do not think that the \dega is overfit to \leadingones. 

As a second contribution, we will describe several realizations of a $(2+1)$ DEGA. One is kept as simple as possible, and we use this for our runtime analysis. It is designed to make understanding the principle as easy as possible, but lacks some natural algorithmic tricks that one would apply in practice. In a second step, we will describe some such amendments. They will not make a substantial difference for the performance on \leadingones, but will help the algorithm to be more robust, for example in the presence of very small fitness improvements, or in cases like the \jump function where the mutation mask from uniform crossover is irreducible.

Our third contribution is to provide runtime simulations. First we investigate \leadingones and confirm the theoretical runtime of $\tilde O(n^{5/3})$, where the tilde indicates that we omit log factors. The exact value of $\lambda$ does not seem to matter too much for \leadingones. Then we compare several versions of the \todega with other basic algorithms like the \tpoga, the \ollga and the \UMDA on a wider class of benchmarks: \textsc{OneMax}, \textsc{LeadingOnes}, \textit{Linear Functions} with harmonic weights and \textit{Maximum Independent Set (MIVS)}. Not surprisingly, the \dega is quite superior on \leadingones compared to the other algorithms that we test. But we find that the \dega is also generally competitive on other benchmarks, especially for moderate $\lambda = \log n$, though here it doesn't flatly outcompete all other algorithms. We emphasize that we did not try to optimize any hyperparameters of the \dega based on the simulation results (but neither for the other algorithms), except for showing several values of $\lambda$. This makes us confident that the \dega can provide benefits in many situations without a general decrease in performance.\smallskip

\noindent\textbf{Larger population sizes.} In this paper we focus on a small population size $\mu =2$, which already requires an interesting analysis. However, a natural next step will be to extend the approach to larger population sizes. Since larger population sizes are associated with larger diversities~\cite{lengler_diversity,opris2024tight}, it is tempting to assume that our paradigm loses strength for larger population sizes. However, we give a heuristic argument why this may not be the case.

First of all, why is larger population size associated with higher diversity? A nice answer to this question was given in~\cite{lengler_diversity}, where the population diversity with and without crossover was analyzed in the situation \emph{without selective pressure}, for a flat fitness function. The diversity was defined as the average Hamming distance in the population. Firstly, it was observed that crossover does not have any influence on the resulting population diversity. Secondly, the diversity at equilibrium was analyzed. This equilibrium diversity indeed increases with $\mu$, so diversity maxes out at a larger value for larger $\mu$. However, the assumption of a flat fitness function is crucial here. As we will argue below, exploitation (fitness improvements) decrease diversity. So the key question is not at what value the diversity maxes out if fitness improvements never occur. The relevant question is rather: if there are frequent fitness improvements that reduce diversity, \emph{how fast} does the diversity recover? It was shown in~\cite{lengler_diversity} that reducing diversity can be  much faster than gaining diversity. In the following, we try to quantify this effect by a thought experiment.

Consider the case of a diverse population in which one search point $x^1$ has a particularly important gene that the other individuals of the population do not have, for example obtained by a mutation. If $x^1$ was to spread its fitness advantage to the rest of the population, assume that it performs crossover with all the $\mu-1$ other individuals (not necessarily all in the same generation, but over time), each time transferring the gene and replacing the other parent with the offspring. We choose this setting, rather than a single generation, because the \emph{average} distance of the population from the optimum is reduced by one, and the population average gives arguably a good scaling when comparing different population sizes. In this setting, the Hamming distance of $x^1$ to all the new individuals is cut by half. But also, if $y^2$ and $y^3$ are the offspring replacing the parents $x^2$ and $x^3$ then their expected Hamming distance is cut by half as well $\EE[H(y^2,y^3)] = \tfrac{1}{2}\EE[H(x^2,x^3)]$. This is because for any bit in which $x^2$ and $x^3$ differ, one of them also differs from $x^1$, so the bit has a probability of $1/2$ to be lost during crossover. Hence, in this (admittedly simplistic) scenario, spreading a single new gene to the rest of the population costs $1/2$ of the total diversity, just as for the $\mu = 2$ case. We leave it to future work to validate or falsify this heuristic reasoning.\smallskip

\noindent\textbf{Related Work.} 
There is a vast amount of work showing the benefits of crossover. We refer to~\cite{sudholt2020benefits} for a comprehensive review of theoretical work and to~\cite{doerr2024runtime} for a discussion of more recent work, and only mention a few landmark results. Crossover can not only help for tailored benchmarks like hierarchical if-and-only-if~\cite{watson2001analysis} and royal-road functions~\cite{watson2007building,jansen2005real}, but also for natural applications like the closest string problem~\cite{sutton2021fixed} and for generic benchmarks like \jump~\cite{jansen2002analysis,whitley2018exploration}. While a main features of all these examples are local optima that crossover helps overcoming, crossover can also speed up hillclimbing tasks like \onemax~\cite{sudholt2012crossover,corus2020benefits} and \leadingones~\cite{cerf2024population}, albeit to a lesser extent. 

A common theme of many of these examples is that crossover works better the larger the population diversity is. Especially the extensive theoretical work on the \jump problem~\cite{jansen2002analysis,kotzing2011crossover,dang2017escaping,doerr2024runtime,opris2024tight} can be regarded as a quest for proving better guarantees for the population diversity, or for finding parameters or tweaks of the algorithms that increase diversity. In particular, explicit diversity-enhancing mechanisms have been shown to work extremely well in this situation~\cite{dang2017escaping}. However, the \jump function sidesteps a central dilemma of crossover: exploitation of crossover generally reduces diversity, as we have described above. For \jump, this does not play a role since the crossover step ends in the optimum, after which the algorithm stops. In practice, it seems more realistic that crossover should help many times, not just once. This is why our main focus is the \leadingones problem, where many equally hard steps follow each other, and crossover may be useful for each of them.

Introduced by Rudolph, the function \leadingones is one of the most classical theoretical benchmarks for evolutionary algorithms~\cite{rudolph1997convergence}. Even simple algorithms like random local search or the \ooea find the optimum with $\Theta(n^2)$ function evaluations~\cite{droste2002analysis}. However, this quadratic runtime bound turned out to be a sturdy barrier. Lehre and Witt could show that no unbiased\footnote{Despite its name our crossover operator is unbiased in the sense of Lehre and Witt.} mutation-based algorithm can optimize \leadingones in less than quadratic time~\cite{lehre2012black}, or in other words, that the \emph{unary unbiased black-box complexity} of \leadingones is $\Omega(n^2)$. Hence, for unbiased operators this barrier can only be broken by crossover and other recombination mechanisms. This was later extended by Doerr and Lengler to the $(1+1)$ elitist black-box complexity~\cite{doerr2017introducing}, which is also $\Omega(n^2)$ without the unbiasedness condition~\cite{doerr2018elitist}. For larger population sizes, it was known that $\Omega(n^2)$ is not a principled barrier. The binary unbiased black-box complexity of \leadingones is $O(n\log n)$~\cite{doerr2011faster}, and becomes even lower for higher arities~\cite{doerr2012black,afshani2019query}. In principle, the proofs in~\cite{doerr2011faster,doerr2012black,afshani2019query} give algorithms that optimize \leadingones fast. However, those algorithms are strongly tailored to \leadingones. Most of them are highly artificial and would immediately break when applied to a different problem. A notable exception is a basic algorithm in~\cite{doerr2011faster} with runtime $O(n\log n)$, which contains related ideas if viewed from the right angle. We will explain this in more detail below and also include a version of this algorithmic idea into our runtime simulation. However, the result from~\cite{doerr2011faster} has remained a purely theoretical result. The verbatim formulation from~\cite{doerr2011faster} only works for \leadingones, and we are not aware of attempts to turn the underlying idea into a practical general-purpose optimizer. To our best knowledge, the $\Omega(n^2)$ barrier has not been broken by any practical population-based algorithm.\footnote{This is to be understood for the class of all automorphic transformations of \leadingones, which allow different target strings and different orderings of the positions. Our runtime bounds of $\tilde O(n^{5/3})$ holds for this whole class. There are certainly algorithms which can optimize the \leadingones function itself in sub-quadratic time, for example, by starting with the all-ones string, or by testing the bits in sequential order. However, those algorithm do not transfer to automorphic transformations of \leadingones.}  In particular, it has been explicitly shown by Cerf and Lengler that diversity-enhancing mechanisms can improve the runtime of the standard $(2+1)$ GA by a constant factor, but not beyond the regime of $\Theta(n^2)$~\cite{cerf2024population}. However, the $\Theta(n^2)$ barrier \emph{has} been broken by the significance-based cGA~\cite{doerr2020significance}, which optimized \leadingones in time $O(n\log n)$. Despite its name, this algorithm is not a genetic algorithm and is not based on populations at all, but it is rather an Estimation-of-Distribution Algorithm which works with statistical properties of the fitness landscape. Together with our work, this may be an additional indication that population-based algorithms may still not have reached their full potential.

The idea of diversity-preserving exploitation of crossover bears some resemblance to the idea used in the \ollga~\cite{doerr2015black}. In that algorithm, $\lambda$ mutations of the same parent $x$ are generated, then the fittest $y$ of those $\lambda$ offspring is selected (which is usually less fit than the parent), and then $\lambda$ biased crossover between $x$ and $y$ are performed. This uses the same of idea of hoping that a biased crossover retains the most beneficial mutation. However, the idea is applied in a very different situation and with a very different aim. In the \ollga, $y$ is a mutation of $x$ that is typically less fit, and the aim is to find a new search point that is fitter than both $x$ and $y$. In our paradigm, $y$ is fitter than $x$, and we accept that the biased crossover may be less fit than $y$, as long as it is fitter than $x$ and has low Hamming distance of $x$ to preserve population diversity. Not surprisingly, practically the differences to the \ollga are very large, for example that it works with population size one. \smallskip

\noindent\textbf{Algorithms design from theoretical benchmarks.} Traditionally, the theory community in the field is very cautious with suggesting new algorithms from evidence on theoretical benchmarks. This has good reasons: there is a substantial risk of overfitting the algorithm to the benchmark. It is rather easy to write down an algorithm which works well on one specific benchmark like \leadingones, but which fails on other benchmarks. Therefore, most runtime analysis has concentrated on understanding algorithms that are already established in practical work, thus avoiding the overfitting problem. The community thus expects high standards of algorithmic suggestions that stem from runtime analyses. We propose an explicit formulation of such standards through the following conditions.
\begin{itemize}
    \item \emph{Generality:} The underlying mechanism must be general enough to plausibly work in a wide range of situations.
    \item \emph{Transparency:} It must be clear in which situation the mechanism gives improvement, and the reason why it gives improvements must be well-understood.
    \item \emph{Impact:} At least on one benchmark the improvement is substantial, and this should be clearly demonstrated by proof or by experiment. Ideally, the benchmark is not specifically designed for the problem, but has been studied before. 
\end{itemize}
Despite the general caution, there are several algorithmic ideas that have made their way from theory to algorithmic portfolios. This includes the \ollga~\cite{doerr2015black}, heavy-tailed mutation operators as in the so-called \emph{fast EAs}~\cite{doerr2017fast} and self-adjustment schemes like the one-fifth rule for discrete search spaces~\cite{doerr2021self}. We believe that our results for the diversity-preserving
exploitation of crossover introduced in this paper meets those standards as well, and that this justifies further systematic investigations of this paradigm. 


\section{The DEGA Algorithm}\label{sec:algo}
As announced, we first describe a simple version of the algorithm that we will use for the runtime analysis on \leadingones. This is designed to be a minimal algorithm to show improvement, with as few components as possible and simpler than the general DiPEC framework that we have described above.\footnote{We believe that it is generally beneficial to investigate such simplified algorithms. Even if the algorithms in practice contain a lot of additional components, this strategy makes it more transparent which components are responsible for the runtime improvement. The alternative is to use a complex algorithm together with extensive ablation studies.} 
In a second step we then list extensions that are likely beneficial for making the algorithm more robust and more widely applicable.

The simple \todega is described in \Cref{algo:dega}. It starts with two random antipodal search points.\footnote{It seems natural for the algorithm to start with maximal diversity. It simplifies the analysis, but we don't believe that it is crucial for the runtime result.} 
When the two search points in the population have the same fitness, the algorithm simply uses mutation to either find a better search point or to increase diversity, lines \ref{algoline:mutation}--\ref{algoline:end_of_mutation}.  The operator ``mutate'' in line \ref{algoline:mutation} is standard bit mutation with mutation rate $1/n$, i.e. it flips each bit independently with probability $1/n$. The ``SelectPopulation'' operator in line \ref{algoline:Select} selects for fitness, with diversity as a secondary criterion. In other words, the offspring $y$ is always selected if it is strictly fitter than the two (equally fit) old search points $x^1,x^2$, and always discarded when it is strictly less fit.  When $y$ has the same fitness as $x^1$ and $x^2$, the new population is the pair of points with largest Hamming distance among those $3$ points, with ties broken randomly.  

Whenever the two search points have different fitnesses then the algorithm uses the DiPEC paradigm to replace the less fit $x^1$ by a biased crossover offspring $y$ that has higher fitness, but small Hamming distance from $x^1$, thus preserving  most of the population diversity (which here is simply the Hamming distance from $x^2$). This is done by the operator Crossover($x^1, x^2, 1/\lambda$), which produces an offspring by taking every bit from $x^2$ with probability $1/\lambda$, and taking all other bits from $x^1$. For \leadingones we can implement the DiPEC paradigm in a simplified way compared to the introduction, lines \ref{algoline:DiPEC_start}--\ref{algoline:replacex1} of \Cref{algo:dega}. We describe a full version later that works on other functions as well. Since on \leadingones we eventually find a crossover offspring as desired, we can simply perform biased crossover until we find such an offspring. Secondly, we directly perform a biased crossover of $x^1$ and $x^2$ in line~\ref{algoline:crossover2}, which selects each bit from $x^1$ with probability $1-1/\lambda$ and from $x^2$ with probability $1/\lambda$. The alternative is to go through an intermediate unbiased crossover $y$ and check first whether that is fitter than the less fit parent. This step can be skipped for \leadingones since there it always succeeds with large probability $1/2$, but it should not be skipped in the general-purpose version of this algorithm, see below.

\begin{algorithm}
    \caption{Algorithm 1: \textbf{\todega}$(n, \lambda)$ for maximizing $f:\{0,1\}^n\to\RR$}\label{algo:dega}
    \begin{algorithmic}[1]
        \State $x \leftarrow$ Uniform($n, \frac 1 2$) 
        \State $P \leftarrow \{x, \bar{x}\}$

        \Repeat
        \State $x^1, x^2 \leftarrow P$
        \If{$f(x^1) = f(x^2)$}\label{algoline:start_of_mutation} \Comment{Enhance Diversity}
            \State $y =$ mutate($x^i$), $i \in \{1,2\}$ random\label{algoline:mutation}
            \State $P \leftarrow$ SelectPopulation($x^1,x^2,y$)\label{algoline:Select}\label{algoline:end_of_mutation}
        \Else \Comment{Exploitation}
            \If {$f(x^1) > f(x^2)$} \label{algoline:swap}
            \State swap $x^1,x^2$ \Comment{Ensure $f(x^1) < f(x^2)$}
            \EndIf
            \Repeat \label{algoline:DiPEC_start}
            \State $y =$ Crossover($x^1, x^2, 1/\lambda$)\label{algoline:crossover2}\Comment{Biased Crossover}
            \Until {$f(y) > f(x^1)$}\label{algoline:DiPEC_end}
            \State $x^1 \leftarrow y$ \label{algoline:replacex1}
        \EndIf
        \Until {termination criterion met}
    \end{algorithmic}
\end{algorithm}

In our analysis, we will split the run of the algorithm into \emph{diversity phases} for $f(x^1)=f(x^2)$ and \emph{exploitation phases} for $f(x^1) \neq f(x^2)$. During a diversity phase the \todega uses mutation, and such phase lasts until the first strict improvement is found. Note that selection favors diversity if there are fitness ties, measured by the Hamming distance of the two search points. Hence, the diversity can only increase during a diversity phase. An exploitation phase lasts until two search points of equal fitness are found, and no mutation is used in such a phase. Instead, for the less fit search point, which we assume to be $x^1$, the algorithm tries to find a fitter search point $z$ close to $x^1$ by biased crossover. Hence, the diversity can only \emph{decrease} during an exploitation phase. Note that it can happen that $x^1$ improves so much that it becomes fitter than $x^2$, in which case the exploitation phase continues. In fact, we will show that for \leadingones this is the most common case, and that a single exploitation phase consists of many improvements of both search points before reaching a population of equal fitness. Thus a diversity phase may consist of many runs through lines \ref{algoline:start_of_mutation}--\ref{algoline:end_of_mutation}, and an exploitation phase may consist of many runs through lines \ref{algoline:swap}--\ref{algoline:replacex1}.   \smallskip

\noindent\textbf{Amendments.} As previously described, the previous algorithm is designed to be a minimal example that is efficient on \leadingones. However, there are several extensions that are natural to make the algorithm robust and avoid obvious failure modes on other benchmarks. We have already mentioned that in general, we would not let the loop in line~\ref{algoline:DiPEC_start}--\ref{algoline:DiPEC_end} run indefinitely, but rather terminate it after a while, for example after at most $\lambda$ steps. Before even entering the loop, we should not skip the step of creating an \emph{unbiased} crossover $y$ of $x^1$ and $x^2$ and checking whether it is fitter than $x^1$. Only in that case should we start creating biased crossover $y'$ between $y$ and $x^1$. In fact, this idea can be iterated in order to reduce the cost of finding one offspring with the DiPEC paradigm from $\lambda+1$ to $O(\log \lambda)$. We describe this method as algorithm $A_{\text{BB}}$ in Section~\ref{sec:experiments}. However, while this algorithm works well on \leadingones, it does not show good performances on other benchmarks.

Apart from that, the decision to \emph{exclusively} use crossover in an exploitation phase works for \leadingones, but bears the risk of getting stuck when biased crossover does not work. Conversely, for general fitness landscapes crossover may also be beneficial if both parents have the same fitness. Hence, for a practical implementation we suggest that in any case the algorithm performs randomly either a mutation or a uniform crossover (or a combination of both). In case of a crossover $y$, the algorithm enters the exploitation phase only if $y$ is strictly fitter than the less fit offspring $x^1$ (where ties between $x^1$ and $x^2$ are broken randomly), in which case it performs $\lambda$ biased crossover between $x^1$ and $y$. We consider this the most crucial amendment to turn the \todega into a practical algorithm. We do not believe that this would change the runtime on \leadingones, but it would make the analysis more complicated.

There are some other extensions which seem natural, but we leave their exploration for future work. For continuous optimization, but also for discrete optimization with many fitness levels, the condition $f(x^1)=f(x^2)$ may not be satisfied often. Hence, it may make sense to replace this by a soft condition, for example that the difference $|f(x^1)-f(x^2)|$ is below its average value over the last $n$ generations. Likewise, the criterion $f(z) > f(x^1)$ for replacing $x^1$ in line~\ref{algoline:replacex1} may be too soft in such cases. Every replacement costs diversity, so we want to avoid replacements with negligible fitness improvements. Hence, we may replace the condition by a stronger one, 
for example that $z$ must retain at least $10\%$ of the fitness advantage that $x^2$ has over~$x^1$.  

Since it is not clear a priori how to choose the parameter $\lambda$, it would also be natural to adapt this parameter dynamically, for example by starting for each new population with a large $\lambda = \lambda_0$ (e.g., $\lambda_0 = n^{2/3}$), but if after $\lambda$ biased offspring no improvement of $x^1$ was found, then $\lambda$ is reduced by a factor of $1/2$. Mind that a smaller $\lambda$ increases the chances of keeping important bits from $x^2$, but costs more diversity. The whole scheme increases the runtime only by a constant factor of roughly $2$, but removes a parameter choice. Moreover, this variant also has a chance of eventually accepting an unbiased crossover and thus be more efficient on benchmarks like \jump, where uniform crossover is beneficial over biased crossover. 

Finally, the algorithm extends naturally to larger populations. In this case, we would choose two random parents for crossover, either uniformly or by methods like tournament selection. Then uniform crossover is performed, and if the offspring $y$ is fitter than the worse parent $x$ then $\lambda$ biased offspring between $x$ and $y$ are performed, the best of which may replace $x$. Instead, one could also replace an unrelated third search point $x'$, e.g., the least fit in the population, by the fittest of $H(y,x')$ biased crossovers between $y$ and $x'$. We leave exploration of larger populations to future work.

\section{Preliminaries}\label{sec:prelim}
The main result of this paper is on the \textsc{LeadingOnes} benchmark. The function \textsc{LeadingOnes} or $\LO: \{0, 1\}^n \to \{0,\ldots,n\}$ counts the number of initial consecutive 1s in a bit string \( x \in \{0, 1\}^n \), formally defined as
$\textsc{LO}(x) = \sum_{i=1}^{n} \prod_{j=1}^{i} x_j$.

\noindent The \emph{runtime} $T$ on \leadingones is the number of function evaluations until the optimum is found. We measure the \emph{time} generally by the number of fitness evaluations and define a \emph{generation} as the time period of creating and evaluating one search point (either in line~\ref{algoline:mutation} or in line~\ref{algoline:crossover2} of \Cref{algo:dega}). 
As explained in \Cref{sec:algo}, a run of the \todega on \leadingones alternates between diversity and exploitation phases, where the latter can cover many fitness improvements. Starting with $k=0$ we denote by $P_k =\{x^1_k, x^2_k\}$ the population at the beginning of the $k$-th phase, where $\LO(x^1_k) \le \LO(x^2_k)$. The $i$-th bit of $x^1_k$ is denoted as $(x^1_k)_i$, and similarly for other search points. Sometimes we also want to refer to variables at some time $t$ within a phase, and by overlay of notation we denote this by an index $t$, for example $P_t$ is the population in generation $t$. We will use the indices in a way that no confusion between the two options should arise.
 
At any point during a run we split the $n$ positions into an \emph{optimized} part and a \emph{non-optimized} part. The optimized part is given by the first $\LO(x^2) = \max \{ \LO(x^1), \LO(x^2)\}$ bits (i.e., optimized by the fitter search point $x^2$). The non-optimized part $\NO$ is given by the remaining bits, but in a diversity phase (if both search points have the same fitness) we exclude the bit at position $\LO(x^2)+1$, which we we call the \emph{critical bit}. Hence, $\NO=\{\LO(x^2)+1,\ldots,n\}$ in an exploitation phase and $\NO=\{\LO(x^2)+2,\ldots,n\}$ in a diversity phase. Throughout the paper we denote by $\alpha := |\NO|/n$ the fraction of the non-optimized part in the total string. 
%
%
%
As before, 
$\alpha_k$ is the value of $\alpha$ at the beginning of the $k$-th phase. We will focus specifically at the Hamming distance $H_k$ between the non-optimized parts of our two search points at the beginning of the $k$-th phase, as well as its complement,
\begin{align*}
    H_k := \sum\nolimits_{i \in \NO}\big(x_k^1\big)_i \oplus \big(x_k^2\big)_i \quad \text{and} \quad \bar H_k \coloneqq \alpha_k n - H_k,
\end{align*}
where $\oplus$ denotes XOR.
Note that the fraction $\alpha$ of the non-optimized part can only decrease over the course of a run. 


Finally we introduce some notation for different bits. We have already defined the \emph{critical bit} at position $\LO(x^2)+1$, where $\LO(x^2)\ge \LO(x^1)$. Note that this bit is always zero in $x^2$. In a diversity phase, since $\LO(x^1) = \LO(x^2)$ the critical bit is also zero in $x^1$. In an exploitation phase, the position of the first zero-bit in $x^1$ is called the \emph{improving position} or \emph{improving bit}. Note that this is a one-bit in $x^2$. In the non-optimized part of size $\alpha n$, a \emph{blocking bit} is a bit that is zero in both $x^1$ and $x^2$, and a \emph{skipping bit} is a bit that is one in both $x^1$ and $x^2$. These bits are also illustrated in \Cref{fig:bits}. We denote the number of blocking and skipping bits at the beginning of the $k$-th phase with $B_k$ and $S_k$ respectively. Note that $\bar H_k = B_k+S_k$. 
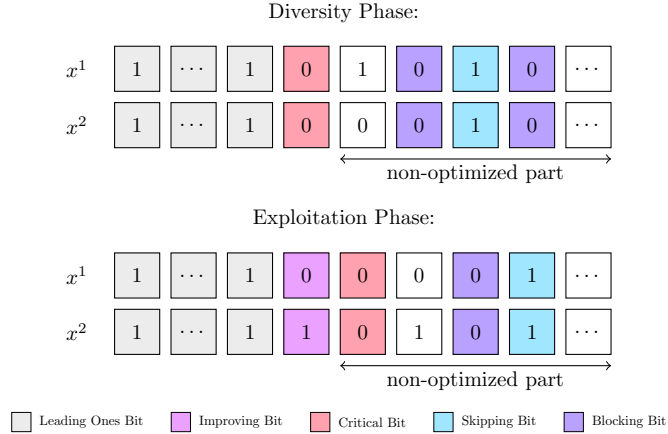
\begin{figure}[ht]
    \centering
    \begin{tikzpicture}[scale=1.0, every node/.style={scale=0.8}]
        \node at (3.75, 0.75) {Diversity Phase:};
        \node at (3.75, -2) {Exploitation Phase:};

        \node at (0.2, 0) {$x^1$};
        \node at (0.2, -0.75) {$x^2$};

        \draw [<->] (3.7,-1.2) -- (7.3,-1.2);
        \node at (5.5, -1.4) {non-optimized part};
        
        \node at (0.2, -2.75) {$x^1$};
        \node at (0.2, -3.5) {$x^2$};

        \draw [<->] (3.7,-3.95) -- (7.3,-3.95);
        \node at (5.5, -4.15) {non-optimized part};

        \foreach \i in {0, 0.5, 1} {
            \node[fill=gray!15, draw=none, rectangle, minimum size=0.75cm] at (\i*1.5 + 1, 0) {};
            \node[fill=gray!15, draw=none, rectangle, minimum size=0.75cm] at (\i*1.5 + 1, -0.75) {};
        }

        \foreach \i in {2.5,3.5} {
            \node[fill=myblue!50, draw=none, rectangle, minimum size=0.75cm] at (\i*1.5 + 1, 0) {};
            \node[fill=myblue!50, draw=none, rectangle, minimum size=0.75cm] at (\i*1.5 + 1, -0.75) {};
        }

        \foreach \i in {1.5} {
            \node[fill=myred!50, draw=none, rectangle, minimum size=0.75cm] at (\i*1.5 + 1, 0) {};
            \node[fill=myred!50, draw=none, rectangle, minimum size=0.75cm] at (\i*1.5 + 1, -0.75) {};
        }

        \foreach \i in {3} {
            \node[fill=mycyan!50, draw=none, rectangle, minimum size=0.75cm] at (\i*1.5 + 1, 0) {};
            \node[fill=mycyan!50, draw=none, rectangle, minimum size=0.75cm] at (\i*1.5 + 1, -0.75) {};
        }
        
        \foreach \i/\val in {0/1, 0.5/$\dots$, 1/1, 1.5/0, 2/1, 2.5/0, 3/1, 3.5/0, 4/$\dots$} {
            \node[draw, rectangle, minimum size=0.75cm] (b1d-\i) at (\i*1.5 + 1, 0) {\val};
        }

        \foreach \i/\val in {0/1, 0.5/$\dots$, 1/1, 1.5/0, 2/0, 2.5/0, 3/1, 3.5/0, 4/$\dots$} {
            \node[draw, rectangle, minimum size=0.75cm] (b2d-\i) at (\i*1.5 + 1, -0.75) {\val};
        }

        \foreach \i in {0, 0.5, 1} {
            \node[fill=gray!15, draw=none, rectangle, minimum size=0.75cm] at (\i*1.5 + 1, -2.75) {};
            \node[fill=gray!15, draw=none, rectangle, minimum size=0.75cm] at (\i*1.5 + 1, -3.5) {};
        }

        \foreach \i in {3} {
            \node[fill=myblue!50, draw=none, rectangle, minimum size=0.75cm] at (\i*1.5 + 1, -2.75) {};
            \node[fill=myblue!50, draw=none, rectangle, minimum size=0.75cm] at (\i*1.5 + 1, -3.5) {};
        }

        \foreach \i in {2} {
            \node[fill=myred!50, draw=none, rectangle, minimum size=0.75cm] at (\i*1.5 + 1, -2.75) {};
            \node[fill=myred!50, draw=none, rectangle, minimum size=0.75cm] at (\i*1.5 + 1, -3.5) {};
        }

        \foreach \i in {1.5} {
            \node[fill=mymagenta!50, draw=none, rectangle, minimum size=0.75cm] at (\i*1.5 + 1, -2.75) {};
            \node[fill=mymagenta!50, draw=none, rectangle, minimum size=0.75cm] at (\i*1.5 + 1, -3.5) {};
        }

        \foreach \i in {3.5} {
            \node[fill=mycyan!50, draw=none, rectangle, minimum size=0.75cm] at (\i*1.5 + 1, -2.75) {};
            \node[fill=mycyan!50, draw=none, rectangle, minimum size=0.75cm] at (\i*1.5 + 1, -3.5) {};
        }

        \foreach \i/\val in {0/1, 0.5/$\dots$, 1/1, 1.5/0, 2/0, 2.5/0, 3/0, 3.5/1, 4/$\dots$} {
            \node[draw, rectangle, minimum size=0.75cm] (b1e-\i) at (\i*1.5 + 1, -2.75) {\val};
        }

        \foreach \i/\val in {0/1, 0.5/$\dots$, 1/1, 1.5/1, 2/0, 2.5/1, 3/0, 3.5/1, 4/$\dots$} {
            \node[draw, rectangle, minimum size=0.75cm] (b2e-\i) at (\i*1.5 + 1, -3.5) {\val};
        }

        \node at (3.75, -4.7) { 
            \begin{tikzpicture}
                \matrix[matrix of nodes, nodes={align=left}, column sep=0.1cm] (m) {
                    \node {
                        \begin{tikzpicture}[baseline=(b.center)]
                            \node[rectangle, minimum size=0.5cm, fill=gray!15, draw] (b) {};
                            \node[right=0.35cm] {Leading Ones Bit};
                        \end{tikzpicture}
                    }; &
                    \node {
                        \begin{tikzpicture}[baseline=(b.center)]
                            \node[rectangle, minimum size=0.5cm, fill=mymagenta!50, draw] (b) {};
                            \node[right=0.35cm] {Improving Bit};
                        \end{tikzpicture}
                    }; &
                    \node {
                        \begin{tikzpicture}[baseline=(b.center)]
                            \node[rectangle, minimum size=0.5cm, fill=myred!50, draw] (b) {};
                            \node[right=0.35cm] {Critical Bit};
                        \end{tikzpicture}
                    }; &
                    \node {
                        \begin{tikzpicture}[baseline=(b.center)]
                            \node[rectangle, minimum size=0.5cm, fill=mycyan!50, draw] (b) {};
                            \node[right=0.35cm] {Skipping Bit};
                        \end{tikzpicture}
                    }; & 
                    \node {
                        \begin{tikzpicture}[baseline=(b.center)]
                            \node[rectangle, minimum size=0.5cm, fill=myblue!50, draw] (b) {};
                            \node[right=0.35cm] {Blocking Bit};
                        \end{tikzpicture}
                    }; \\
                };
            \end{tikzpicture}
        };
    \end{tikzpicture}
    \caption{Different bits during optimization}
    \label{fig:bits}
\end{figure}

Due to the symmetries of \leadingones as discussed in~\cite{cerf2024population}, for a given Hamming distance of $x^1$ and $x^2$ in the non-optimized part, $\bar H_k$ positions in $\NO$ in which $x^1$ and $x^2$ coincide are located uniformly at random.  
More precisely, the following statements have been formally proven in~[2, Corollary~7] for all positions except for the critical bit in the exploitation phase, which we cover in Lemma \ref{lem:critical} below. For a given Hamming distance of $x^1$ and $x^2$ in the non-optimized part, $\bar H_k$ positions in $\NO$ in which $x^1$ and $x^2$ coincide are located uniformly at random. Moreover, except for the critical bit in an exploitation phase, the non-optimized part of $x^1$ and $x^2$ are distributed uniformly at random (not independent between $x^1$ and $x^2$) among all strings of length $|\NO|$. In particular, for $x^2$ each of those positions has independently a probability of exactly $1/2$ to be a zero-bit. In particular, when a diversity phase results in a fitness improvement then each following bit has independently a chance of $1/2$ to also be a one-bit, also called a \emph{free rider}. Hence, the increase in fitness is dominated by a distribution $1+\geo(1/2)$, where $\geo$ denotes a geometric distribution.\footnote{There are two conflicting conventions for the geometric distribution, shifted by $+1$. We use the convention that the support of the distribution includes the value $0$. } It follows the distribution exactly if we truncate the geometric distribution at $\alpha n -1$.

\begin{lemma}\label{lem:critical}
For a run of the \dega on \leadingones and for $i\in[n]$, let $t_1$ be the first generation in which $\LO(x_{t}^2) = i$ (assuming this event happens), and let $t_2+1$ be the first generation in which $\LO(x_t^1) \ge i$. Let $\NO$ be the non-optimized part in generation $t_2$, and let $H_{t_2+1}$ be the Hamming distance of $x_{t_2+1}^1$ and $x_{t_2+1}^2$ in $\NO$.\footnote{The slight mismatch of indices in the definition of $\NO$ and $H_t$ is needed for a correct statement. It will not play a role later since $H_{t}$ does not change too much in a single generation.} Then the probability that in generation $t_2+1$ the $(i+1)$-th bit differs in $x^1$ and $x^2$ is 
\begin{align}\label{eq:critical}
\Pr\big[\large(x^1_{t_2+1}\large)_{i+1} \neq \large(x^1_{t_2+1}\large)_{i+1}\big] = \frac{H_{t_2+1}}{|\NO|}.
\end{align}
\end{lemma}
Note that in \Cref{lem:critical}, at time $t_2 + 1$ the old search point $x_{t_2}^2$ stays in the population and becomes the new search point $x_{t_2+1}^1$, because it is now at most as fit as the offspring $y_{t_2}$ generated in generation $t_2$. In case of equal fitness it could be either labeled $x_{t_2+1}^1$ or $x_{t_2+1}^2$, but let us assume the former for consistency. In any case, this search point has a zero-bit at position $i+1$. Hence, the event in~\eqref{eq:critical} could equivalently be phrased as ``$(x^2_{t_2+1})_{i+1} = 1$'' or as ``$\LO(x^2_{t_2+1}) \ge i+1$'' or as ``the population $P_{t_2+1}$ contains a one-bit at position $i+1$''. 
\begin{proof}[Proof of \Cref{lem:critical}]
First note that before generation $t_1-1$, none of the bits in $\{i+1,\ldots,n\}$ has any impact on the fitness of explored search points. Hence, the only impact of this part on selection is by the Hamming distance between the two search points. Therefore, at any time $t< t_1-1$ the positions at which $x^1_{t}$ and $x^2_{t}$ differ in $\NO$ are uniform at random in $\NO$, see also~[2] for a formal argument. 

At time $t_1-1$, the zero-bit at position $i+1$ of the offspring $y_{t_1-1} = x^2_{t_1}$ does influence the fitness of $y_{t_1-1}$, but it does not impact selection, since $y_{t_1-1}$ is fitter than the two old individuals regardless of its value at position $i+1$. Likewise, consider any offspring $y_t$ for $t_1 \le t < t_2$. We have $\LO(y_t) < i$ by definition of $t_2$, and therefore the $i+1$-th has no impact on the fitness of $y_t$. Finally, for $t=t_2$, the $i+1$-th bit of $y_t$ does influence the fitness, but it does not influence selection: due to $\LO(y_t)\ge i$, the offspring $y_t$ is included into the population regardless of the value of the $i+1$-th bit. 

Summarizing, for every generation $t \le t_2$ the value of the $(i+1)$-th bit of the offspring $y_t$ either does not impact its fitness, or it influences the fitness in a way that is irrelevant for selection. As all other bits in the non-optimized part, it only influences selection via the Hamming distance of the search points. Therefore, for any possible population $\{x^1,x^2\}$ with $(x^2)_{i+1} = 0$ and any distance-preserving automorphism $\pi$ of the positions in $\NO$ with $(\pi(x^2))_{i+1}=0$, we have $\Pr[P_{t_2+1} = \{x^1,x^2\}] = \Pr[P_{t_2+1} = \{\pi(x^1),\pi(x^2)\}]$, similar to~[2, Lemma 4]. Now, for any given Hamming distance $H$ in $\NO$ between $x^1$ and $x^2$ we can find an automorphism $\pi$ that maps the $H$ differing bits to arbitrary positions in $\NO$, by swapping positions and possibly flipping the role of ones and zeros to ensure $(\pi(x^2))_{i+1}=0$. Hence, all positions in $\NO$ have the same probability to be different in $P_{t_2+1}$, and this probability must be $H_{t_2+1}/|\NO|$.
\end{proof}

\section{The \todega on \leadingones}\label{sec:proofs}
Our main result is the following upper bound on the expected runtime $\E{T}$ of the \todega on \leadingones. 
\begin{theorem}\label{thm:mainresult}
Let $T$ be the runtime of the \todega with parameter $2\le \lambda = o(n)$ on \leadingones. Then
\begin{align}\label{eq:maintheorem}
    \EE[T] = O\big( \lambda n + n^2\log n/\sqrt \lambda \big).
\end{align}
\end{theorem}
Since the first summand increases and the second decreases in $\lambda$, their sum is asymptotically minimal when both terms are equal. This holds for $\lambda = (n\log n)^{2/3}$, which gives
\begin{align}
    \label{eq:E[T]-final}
    \EE[T] = O\big( n^{5/3}\log^{2/3}n \big) .
\end{align} 

\subsection{Proof of Theorem~\ref{thm:mainresult}}
The rest of this section is devoted to proving \Cref{thm:mainresult}, and throughout the section we will consider the \todega with $2\le \lambda = o(n)$ on \leadingones. Throughout the analysis, $P=\{x^1,x^2\}$ denotes the current population of the algorithm, where $\LO(x^1)\le \LO(x^2)$. The key ingredient is to show that the diversity remains close to its maximum, or equivalently that its complement $\bar H_t$ in the non-optimized part remains bounded. This is formalized by the following theorem. Recall that $\alpha\in[0,1]$ denotes the non-optimized fraction of the string and that we overload the index of $\alpha$:  for the $k$-th phase we refer by $\alpha_k$ to the value of $\alpha$ at the beginning of phase $k$, and for the $t$-th generation (``at time $t$'') we denote its value by $\alpha_t$. We omit the index altogether if the considered point in time is clear from the context.
\begin{theorem}
\label{thm:main}
There is a constant $c>0$ such that for all times $t$ with $\alpha_t \ge n^{-1/3}$, 
\begin{align*}
    \E{\bar H_{t+1}} \leq c \cdot \alpha_t n \log n/\sqrt{\lambda}.
\end{align*}
\end{theorem}
\noindent We first show how \Cref{thm:main} implies~\Cref{thm:mainresult}.

\begin{proof}[Proof of \Cref{thm:mainresult}]
We will prove the seemingly weaker bound 
$\EE[T] = O(n^{5/3} + \lambda n + n^2\log n/\sqrt \lambda)$. Note that the additional term $n^{5/3}$ does nothing: either $\lambda \ge n^{2/3}$ and $n^{5/3}$ is dominated by $\lambda n$; or $\lambda \le n^{2/3}$ and $n^{5/3}$ is dominated by the last term.

In our analysis we will focus on the part with $\alpha_t \ge n^{-1/3}$, or equivalently when the non-optimized part has size $\ge n^{2/3}$. The last part with $\alpha_t < n^{-1/3}$ is negligible: we will show below in \Cref{lma:time_exploit_bit} that the algorithm needs in expectation $O(\lambda)$ function evaluations to improve $x^1$ when $\LO(x^1)\neq \LO(x^2)$, and it needs in expectation $O(n)$ evaluations to find a fitness improvement when $\LO(x^1)= \LO(x^2)$. The former can happen at most $n$ times, the latter at most $n^{2/3}$ times when $\alpha_t < n^{-1/3}$. Hence, the time spent for the last interval is at most $O(\lambda n + n^{5/3})$, as required.\footnote{The analysis of this part could easily be improved, but improvements here would not yield a better bound since the bottlenecks are elsewhere.} Moreover, the same argument shows that all exploitation phases together (of the whole run, not just the last interval) take time $O(\lambda n)$, and each diversity phase takes time $O(n)$. Hence, it will suffice to show that for the remaining part with $\alpha_t> n^{-1/3}$, there are at most $O(n\log n/\sqrt{\lambda})$ diversity phases in expectation.

Recall the notion of blocking bits from Section~\ref{sec:prelim}. We define $\calB_i$ for $1\le i \le n-n^{2/3}$ as the event that the algorithm enters the diversity phase with $\LO(x^1)=\LO(x^2)=i$. Assume that $\LO(x^1) < \LO(x^2)=i$ at some point, because otherwise $\calB_i$ cannot occur. Then consider the unique point in time $t$ when $\LO(x^1_t) < \LO(x^2_t) =i$ and an offspring $y$ is found with $\LO(y) \ge i$. Then the $\bar H_{t+1}$ identical bits of $x^2_t$ and $y$ among the last $\alpha_t n$ positions are distributed uniformly at random by \Cref{lem:critical} and~\cite[Corollary~7]{cerf2024population}. In particular, by \Cref{lem:critical}, for a given value of $H_{t+1}$ the probability that position $i+1$ is one of those positions is $\bar H_{t+1}/(\alpha_t n)$. Since $x_2$ has a zero-bit at position $i+1$, this is equivalent to the event $\calB_i$, and thus
\begin{align}
    \Pr[\calB_i \mid \bar H_{t+1}] \le \frac{\bar H_{t+1}}{\alpha_t n},
\end{align}
where the bound would be an equality if we conditioned on the fitness hitting level $i$ at some point during the run. By the law of total expectation and by \Cref{thm:main},
\begin{align*}
    \Pr[\calB_i] = \mathbb E \lbr \Pr[\calB_i \cond \bar H_{t+1}] \rbr \leq \frac{\EE[\bar H_{t+1}]}{\alpha_t n} \leq c \cdot \log n/\sqrt \lambda.
\end{align*}
Thus the expected number of diversity phases with $\alpha_t > n^{-1/3}$ is $\sum_{i=1}^{n-n^{2/3}} \Pr[\calB_i]  \le c n\log n/ \sqrt{\lambda}$. This concludes the proof.
\end{proof}

In the proof of \Cref{thm:mainresult} we used \Cref{thm:main} and also the following lemma, which we prove now.
\begin{lemma}
    \label{lma:time_exploit_bit}
    Assume $\LO(x^1_t) < \LO(x^2_t)$ for the current population $P_t=\{x^1_t,x^2_t\}$ of the \dega on \leadingones. Let $T_{\text{Improve}}$ be the number of $\LO$-evaluations until $x^1$ is replaced by a fitter search point. Then
    $
        \E{T_{\text{Improve}}} = O(\lambda).
    $
\end{lemma}
    \begin{proof}
        This follows directly from the construction of exploitation phases. We perform biased crossovers between $x^1_t,x^2_t$ until we find an improvement for $x^1_t$. Let $\mathcal{I}$ be the event of moving the improving bit from $x^2$ to $x^1$ in one such biased crossover. Then
        \begin{align*}
            \Pr[\mathcal I] = 1/\lambda \quad \text{and} \quad T_{\text{Improve}} \sim \geo(\Pr[\mathcal I]),
        \end{align*}
        where $\geo$ denotes a geometric distribution. Hence $\EE[T_{\text{Improve}}] = O(\lambda)$.
    \end{proof}

\subsection{Proof of Theorem \ref{thm:main}}
\label{sub:proof_main}
It remains to prove \Cref{thm:main}, and this whole section is devoted to that. Recall that $\bar H_k$ is the number of bits that align in the non-optimized part of $x^1, x^2$ at the start of phase $k$. We start by proving two lemmas which analyze how the  diversity changes in a diversity and in an exploitation phase, respectively. In the following, we frequently use these inequalities:
\begin{align}
    \label{eq:exp-bounds}
    1-x \leq \exp \lb -x \rb \quad \text{and} \quad 1-x \geq \exp \lb -2x \rb, x \in [0,1/2].
\end{align}
Combining both gives $(1-x)^r \le \exp(-xr) \le 1-xr/2$ if $x\in [0,1/2]$, $r\ge 0$ and $xr \le 1$.
\begin{lemma}
    \label{lma:div-in-div-phase}
    For $k\ge 1$, assume that the $k$-th phase is a diversity phase with $\alpha_k n\ge n^{2/3}$. Then 
    \begin{align*}
        \E{\bar H_{k+1} \cond \bar H_k} \in \big[ \tfrac{1}{2}\bar H_k-1, \big(1 - \tfrac{1}{4e^{2}}\big)\bar H_k\big].
    \end{align*}
    \begin{proof}
        We first prove the lower bound. We may neglect any mutation that flips the optimized part, since the offspring will be rejected. Among the remaining mutations, we compare the critical bit $i$ with any fixed non-optimized bit $j$. By symmetry, both are equally likely to be flipped before the other. Hence, the probability that bit $j$ is flipped strictly before the first improving mutation is at most $1/2$. So before that mutation, $\bar H_k$ decreases by at most a factor of $1/2$ in expectation. Finally, in the improving mutation, in expectation at most one additional bit is flipped. 
        
        For the upper bound, let $c_2 = 1 - 1/(4e^{2})$. For a fixed generation, let $p_{k,\text{out}}$ be the probability of leaving phase $k$. Then
        \begin{align*}
            p_{k,\text{out}}  = \tfrac{1}{n} \big( 1 - \tfrac 1 n \big)^{n - \alpha n - 1} \le \tfrac{1}{n}.
        \end{align*}
        Let $Y_k$ be the number of $\LO$-evaluations until leaving phase $k$. For $y = n/2$ we calculate
        \begin{align*}
            \Pr\lbr Y_k \geq y \rbr \geq \lb 1 - p_{k,\text{out}}  \rb^y \stackrel{\eqref{eq:exp-bounds}}{\geq} \exp\lb-2p_{k,\text{out}} y \rb \ge 1/e,
        \end{align*}
        where the second step holds for $n \geq 2$. To calculate the reduction in $\bar H_t$, let us consider the event $\mathcal F_i$ of flipping bit $i$ in a one-bit flip in phase $k$; that is, there is a mutation which flips bit $i$ and no other bit. The probability that a fixed mutation is of this type is at least $\tfrac{1}{n}(1-\tfrac{1}{n})^{n-1} \ge \tfrac{1}{en}$. Moreover, by law of total probability, we have
        \begin{align}
        \label{eq:one-bit-flip}
            \Pr \lbr \mathcal F_i \rbr \ge \Pr \lbr \mathcal F_i \cond Y_k \geq y\rbr\cdot \Pr \lbr Y_k \ge y \rbr.
        \end{align}
        Let us analyze $\Pr \lbr \mathcal F_i \cond Y_k \geq y\rbr$. Since more than $y$ trials only increase the chance of $\mathcal F_i$, we get
        \begin{align*}
            \Pr \lbr \mathcal F_i \cond Y_k \ge y\rbr \ge 1- \big( 1- \tfrac{1}{en} \big)^y
            \stackrel{\eqref{eq:exp-bounds}}{\ge} 1-\exp( -\tfrac 1 {2e} ) \stackrel{\eqref{eq:exp-bounds}}{\ge} \tfrac{1}{4e} .
        \end{align*}
        Plugging our finding back into \Cref{eq:one-bit-flip} gives
        $\Pr [ \mathcal F_i ] \ge  1/(4e^2)$. 
        Thus each one-bit in $\bar H_{k}$ has a chance of at least $1/(4e^2)$ to vanish. We conclude that $\E{\bar H_{k+1}} \leq ( 1 - 1/(4e^2) ) \cdot \bar H_k$, as required.
    \end{proof}
\end{lemma}
\noindent We now analyze how $\bar H_k$ develops during an exploitation phase. We first show in \Cref{lma:length_exploit_whp} that with high probability the progress per exploitation phase is at most $O(\sqrt \lambda \log n)$, where we measure progress as the fitness increase of the fitter of the two individuals. Afterwards, we will derive an upper bound on the increase from $\bar H_k$ to $\bar H_{k+1}$ for an exploitation phase $k$ in \Cref{lma:div-in-exploit}.
\begin{lemma}
    \label{lma:length_exploit_whp}
    Let $\gamma >1$ be constant, $c_{\text{len}} \coloneqq 256\gamma$ and assume $2\leq \lambda = o(n)$. Consider an exploitation phase $k$ with $\alpha_k n \geq n^{2/3}$. Let $L_k$ be the progress made in phase $k$, i.e., the fitness difference between the fitter search point in the first generation of phase $k$ and the fitter search point in the first generation of phase $k+1$. Then with probability $1 - o(n^{-\gamma})$, 
    \begin{align*}
        L_k \leq c_{\text{len}} \cdot \sqrt{\lambda} \log n \eqqcolon L_{\text{max}}.
    \end{align*}
\end{lemma}

    \begin{proof}
    The phase ends as soon as the algorithm reaches a blocking bit as the next critical bit, as defined in \Cref{sec:prelim}. Recall that $B_k$ is the number of blocking bits and $S_k$ is the number of skipping bits in the non-optimized part of the bit-string. Notice that $\bar H_k = B_k + S_k$ and $\EE[S_k \cond \bar H_k] = \EE[B_k\cond \bar H_k] = \frac{1}{2} \bar H_k$, because each bit in $\bar H_k$ independently has a chance of $1/2$ to be either a skipping or a blocking bit. We will distinguish two cases.\\

    \noindent \textbf{Case 1:} $\bar H_k \geq \frac{\alpha_k n}{2}$.\\
    To show an upper bound on the duration of the exploitation phase we will need to show that it is very unlikely to have substantially more $S_k$ than $B_k$ bits. Notice that any bit among $\bar H_k$ comes from $S_k$ or $B_k$ independently with probability $1/2$, so $\EE[B_k] \ge \alpha_k n/4$. Applying a Chernoff bound for $\alpha_k = \Omega(n^{-1/3})$ yields
    \begin{align*}
        \Pr \lbr B_k < \frac{\alpha_k n}{8} \rbr \leq \Pr \lbr B_k \leq \big( 1- \tfrac 1 2 \big)\E{B_k}\rbr
        \leq  \exp{\lb  - \frac{\alpha_k n}{32}\rb} = o(n^{-\gamma}).
    \end{align*}
    In the following let $B_k^\geq$ be the event "$B_k \geq \frac{\alpha_k n}{8}$" and $L_k^>$ be the event "$L_k > L_{\text{max}}$". Then we can bound $\Pr[ L_k^> ]$ by
    \begin{align*}
        \Pr[ L_k^> ] = & \Pr[ L_k^> \cond B_k^\ge ]\cdot \Pr [ B_k^\ge ] + \Pr[ L_k^> \cond \neg B_k^\ge ] \cdot \Pr [ \neg B_k^\ge ]\\
        \le & \ \Pr[ L_k^> \cond B_k^\ge ] + \Pr [ \neg B_k^\ge ] \le \Pr[ L_k^>\cond B_k^\ge ] + o(n^{-\gamma}).
    \end{align*}
    To bound $\Pr[ L_k^>\cond B_k^\ge]$, note that $L_k^>$ implies that there is no blocking bit among the next $L_{\max}$ bits after the critical bit of the beginning of phase $k$. We notice that for each bit we have a blocking-probability of $p_B = B_k/(\alpha_k n) \ge 1/8$. Moreover, blocking bits are negatively associated (for a fixed total number, having a blocking bit in position $i$ decreases the chance of having a blocking bit at position $j\neq i$), hence the bound $\Pr [ L_k^> \mid B_k^{\ge}] \le \lb 1- p_B \rb^{L_{\text{max}}}$  applies~\cite[Property~P2]{joag1983negative}. Not finding a blocking bit in $L_{\text{max}}$ generations is therefore unlikely:
    \begin{align*}
        \Pr [ L_k^> \mid B_k^{\ge} ] \le \lb 1- p_B \rb^{L_{\text{max}}} \le e^{-p_B\cdot   L_{\text{max}}} \stackrel{\lambda \ge 1}{\le} n^{-32\gamma} = o(n^{-\gamma}).
    \end{align*}

    \noindent \textbf{Case 2:} $\bar H_k < \frac{\alpha_k n}{2}$.\\
    For this case we know that there are at least $H_k = \alpha_k n - \bar H_k \geq \alpha_k n/2$ diverse bits, i.e. bits that differ between $x^1$ and $x^2$. We will study the development of $B_k$ throughout the first $\tau = \lambda L_{\max}/16$ generations of phase $k$. We will show that after $\tau$ generations enough blocking bits will have been generated to stop the phase w.h.p. in further $\tau$ generations. Note that since this phase only uses crossover, blocking bits can only be generated, not destroyed.

    Let us first show that w.h.p.\! the progress in the first $\tau$ generations is at most $L_{\max}/2$. In each generation, the offspring is accepted if and only if the improving bit is copied from the fitter search point $x^2$ to $x^1$, which happens with probability $1/\lambda$. Let us call $\sigma$ the number of accepted offspring in the first $\tau$ generations of the phase. Then $\EE[\sigma] = L_{\max}/16$, and by the Chernoff bound $\Pr[\sigma > L_{\max}/8] \le \exp(-L_{\max}/48) = o(n^{-\gamma})$. Hence, we may assume $\sigma \le L_{\max}/8$. Those are the only generations in which progress can be made, namely if the offspring $z$ is strictly fitter than both $x^1$ and $x^2$. In this case, each subsequent bit of $z$ after the critical bit has probability $1/2$ to be a one-bit, hence the progress in that generation is dominated by $1+\geo(1/2)$, where $\geo$ denotes a geometric distribution. (Analogously to the argument for diversity phases in Section~\ref{sec:prelim}.) Hence, the expected progress with the first $\sigma \le L_{\max}/8$ accepted offspring is at most $L_{\max}/4$, and the probability that it exceeds $L_{\max}/2$ is at most $\exp(-L_{\max}/16) = o(n^{-\gamma})$ because the sum of geometric random variables is concentrated~\cite[Theorem~1.10.32]{doerr2020probabilistic}. Hence, we may assume that the progress in the first $\tau$ generations is at most $L_{\max}/2$. 

    Let us pessimistically assume that the phase does not end in the first $\tau$ generations, since otherwise we are done. We write $B_{k,\tau}$ for the number of blocking bits after $\tau$ generations.  Then as argued above, the expected number $\EE[\sigma]$ of accepted offspring is $\tau/\lambda = L_{\max}/16$, and again by the Chernoff bound, we have $\sigma \ge L_{\max}/32$ with probability at least $1-\exp(-L_{\max}/256) = 1-o(n^{-\gamma})$. Hence, we may assume that at least $\sigma \ge L_{\max}/32$ offspring are accepted. Let us call this event $\calE_\sigma$.
    
    We will now study how many new blocking bits are generated under this assumption in the first $\tau$ generations. Let us call this number $B_{\text{new}}$. Consider a diverse pair of bits $(0,1)$ or $(1,0)$. By the condition $\bar H_k < \alpha_k n /2$, there are at least $\alpha_k n/2$ such pairs at the beginning of the phase. In a generation with accepted offspring the crossover operator copies each bit from $x^2$ to $x^1$ with probability $1/\lambda$, thus this is the probability of create equal bits $(0,0)$ or $(1,1)$. Ending up with equal values (either $(0,0)$ or $(1,1)$) in $\tau$ generations therefore happens with probability $1 - ( 1 - 1/\lambda )^{\sigma}$. Moreover, since the two options $(1,0)$ and $(0,1)$ for this position were equally likely to start with, the values $(0,0)$ and $(1,1)$ are also equally likely in the end. This yields an additional factor of $1/2$ that the position yields a stopping bit. Hence, the expected number of new blocking bits is
    \begin{align*}
        \EE[B_{\text{new}} \cond \calE_\sigma] & \ge \frac{\alpha_k n}{2}\cdot \Big( 1 - ( 1 - \tfrac{1}{\lambda} )^{L_{\max}/32} \Big)\cdot \frac{1}{2} \\
        &  
        \begin{cases}
        \stackrel{\eqref{eq:exp-bounds}}\ge \frac{\alpha_k n}{4}\cdot \frac{L_{\max}}{64\lambda} = \frac{\gamma \alpha_k n \log n}{\sqrt{\lambda}} \eqqcolon \psi & \text{if } L_{\max} \le 16\lambda, \\
        \ge \frac{\alpha_k n}{4}\cdot (1-e^{-1/2}) \eqqcolon \psi &  \text{if } L_{\max} > 16\lambda.
        \end{cases}
    \end{align*}
    Moreover, conditioned on $\calE_\sigma$, the random variable $B_{\text{new}}$ dominates the sum of $H_k$ independent indicator random variables, one for each position, where they are independent because their starting values are independent and crossover operates independently on them. Hence, $B_{\text{new}}$ dominates a binomial random variable with expectation $\psi$. By the Chernoff bound, we have $\Pr[B_{\text{new}} < \psi/2] \le \exp(-\Omega(\psi)) \le \exp(-\Omega(n^{1/6})) = o(n^{-\gamma})$, since $\lambda = o(n)$ and $\alpha_k = \Omega(n^{-1/3})$. Therefore, from now on we may assume $B_{\text{new}} \ge \psi/2$. Let us call this event $\calB_{\text{new}}^\geq$.

    Let us call $L_k^2$ the progress in phase $k$ after the first $\tau$ generations. It remains to show that conditional on $\calB_{\text{new}}^\geq$, we have $L_k^2 \le L_{\max}/2$ w.h.p. Note that the converse can only happen if none of the next $L_{\max}/2$ positions after the critical position is a blocking bit. We proceed similarly as in Case~1. 
    Conditional on $\calB_{\text{new}}^\geq$ each position has a probability of at least $\psi/(\alpha_k n)$ to be a blocking bit. As in Case~1, the blocking bits are negatively associated. Therefore, the probability that none of the next $L_{\max}/2$ bits are blocking bits is at most $(1-\psi/(\alpha_k n))^{L_{\max}/2}$, and hence
    \begin{align}
        \Pr\big[L_k^2 > L_{\max}/2 \cond \calB_{\text{new}}^\geq\big] \le \lb 1-\tfrac{\psi}{\alpha_k n}\rb^{L_{\max}/2}.\label{eq:Lk2}
    \end{align}
    If $L_{\max} \le 16\lambda$ then we have $\psi/(\alpha_k n) = \gamma \log n /\sqrt{\lambda}$, and together with $L_{\max} = 256\gamma\sqrt{\lambda} \log n$  the right hand side of~\eqref{eq:Lk2} is at most $\exp(- 128\gamma^2 \log^2 n) = o(n^{-\gamma})$. If $L_{\max} > 16\lambda$ then $\psi/(\alpha_k n)  = (1-e^{-1/2})/4 \ge 1/12$, and the statement follows from $(1-1/12)^{L_{\max}/2}\le \exp(-10\gamma\sqrt{\lambda} \log n) =o( n^{-\gamma})$. This concludes the proof.
    \end{proof}

\noindent Next, we apply \Cref{lma:length_exploit_whp} to derive an upper bound on the expected increase from $\bar H_k$ to $\bar H_{k+1}$.
\begin{lemma}
    \label{lma:div-in-exploit}
    Let $\gamma >1$,  $c_{\text{len}} = 256\gamma$ and $2\leq \lambda = o(n)$. Let phase $k$ be an exploitation phase with $\alpha_k n = \Omega(n^{2/3})$. Then
    \begin{align*}
        \E{\bar H_{k+1} - \bar H_{k}} \leq 5c_{\text{len}} \cdot \alpha_k n \cdot \log n/\sqrt{\lambda} .
    \end{align*}
    \begin{proof}
        As in the proof of \Cref{lma:length_exploit_whp} let $L_k^\leq, L_k^>$ be the events of staying within or exceeding $c_{\text{len}} \cdot \sqrt \lambda \log n$ generations in phase $k$ respectively. Also, let $\Delta \bar H_k \coloneqq \bar H_{k+1} - \bar H_k$. First we notice that a maximum of $\alpha_k n$ bits can contribute to $\Delta \bar H_k$. Therefore,
        \begin{align}
            \EE[\Delta \bar H_k]  & = \EE[\Delta \bar H_k \cond L_k^\leq] \cdot \Pr[ L_k^\leq ] + \EE[\Delta \bar H_k \cond L_k^>] \cdot \Pr[ L_k^> ] \label{eq:negligible_error}\\
            & \leq  \EE[\Delta \bar H_k \cond L_k^\leq]+ \alpha_k n \cdot o(n^{-\gamma}) \leq \EE[\Delta \bar H_k \cond L_k^\leq] + o(n^{1-\gamma}).\nonumber
        \end{align}
        We can choose any $\gamma > 1$, so the latter summand gives us a contribution of only $o(1)$. From now on we will study $\EE[\Delta \bar H_k \cond L_k^\leq]$. 

        Conditioned on $L_k^\leq$, we want to bound the number of improving offspring. Note that every improving offspring increases the fitness of the worse search point by at least one. By slight abuse of notation, let us call $x_k^1$ and $x_k^2$ the two search points at the beginning of the phase. Then $x^1_k$ can improve at most to $\LO(x^2_k)+L_k$, since this is the maximal fitness obtained in this phase. Hence, the number of improving offspring is bounded by $L_k+\LO(x^2_k)-\LO(x^1_k)$. Note that the difference $\LO(x^2_k)-\LO(x^1_k)$ was created by free riders (bits that were 1 by chance) in the generation in which $x^2$ was created as new offspring. Thus it is distributed as $1+\geo(1/2)$ and is bounded by $2\log_2 n$ with probability at least $1-1/n^2$. By the same calculation as in~\eqref{eq:negligible_error}, the error event is so unlikely that it is negligible, and we may thus assume that the number of improving offspring is at most $L_{\max}+2\log_2 n \le 2L_{\max}$ when $n$ is large enough.
        
        We now use a similar calculation as in the proof of \Cref{lma:length_exploit_whp}, but now we compute an upper bound (instead of a lower bound) for the probability of turning two diverse bits into identical bits. Note that here we do not care whether the result is a blocking bit or skipping bit. We call $p_{\text{id}}$ this probability of turning a fixed diverse pair into an identical one in the $k$-th phase. By using \Cref{eq:exp-bounds} and $\lambda \geq 2$ we find
        \begin{align*}
            p_{\text{id}} \leq 1 - \big( 1- \tfrac 1 \lambda \big)^{2L_{\max}} \leq 1 - \exp( - 4L_{\max}/\lambda ) \leq 4L_{\max}/\lambda.
        \end{align*}
        Therefore, recalling the definition of $L_{\max}$ from \Cref{lma:length_exploit_whp},
        \begin{align*}
            \EE\big[\Delta \bar H_k \cond L_k^\geq\big] \leq \alpha_k n \cdot p_{\text{id}} \leq \alpha_k n \cdot  4 c_{\text{len}}  \sqrt{\lambda} \log n/\lambda .
        \end{align*}
        We conclude $\E{\Delta \bar H_k} \leq 4c_{\text{len}} \cdot  \alpha_k n \log n/\sqrt \lambda + o(1)$ in total. 
    \end{proof}
\end{lemma}

Now that we know how the diversity develops in the two respective phases we can finally prove \Cref{thm:main}.
\begin{proof}[Proof of \Cref{thm:main}]
Fix  arbitrarily $\gamma \coloneqq 2$ (any other $\gamma >1$ would also work). Choose $c>0$ large enough such that $(1-\tfrac{1}{4e^2})(1+5c_{\text{len}}/c) < 0.98$, where $c_{\text{len}}$ is the constant from \Cref{lma:length_exploit_whp}.

We first prove by induction over $k$ that 
\begin{align}\label{eq:induction_hypothesis}
    \E{\bar H_k} \leq c \cdot \alpha_k n \log n/\sqrt{\lambda}
\end{align}
holds at the beginning of every exploitation phase, where we do the inductive step from $k$ to $k+2$, i.e., we consider the combined effect of an exploitation and a diversity phase. Afterwards, we show that it also holds in between and with a index-shifted $\alpha$ if we increase the factor $c$, i.e., we show for some constant $c'>0$
\begin{align}\label{eq:in_between}
    \E{\bar H_t} \leq c' \cdot \alpha_{t-1} n \log n/\sqrt{\lambda}
\end{align}
for all times $t$ for which $\alpha_{t-1} \ge n^{-1/3}$. Note that this is the statement of \Cref{thm:main}.

For the start of the induction, note from \Cref{algo:dega} that the algorithm always starts with two antipodal search points, i.e. at maximum diversity. Moreover, this implies that initially $\LO(x^1) \neq \LO(x^2)$, so the first phase is an exploitation phase and $\bar H_k = 0$ for $k=0$. This establishes the base case of the induction.

For the inductive step, assume that~\eqref{eq:induction_hypothesis} holds for some even $k\ge 0$. Recall $L_{\text{max}}$ from \Cref{lma:length_exploit_whp} and let $\calL= L_{k}^\leq \cap L_{k+1}^\leq$  be the event that $L_k \le L_{\text{max}}$ and $L_{k+1} \le L_{\text{max}}$.  Then $\Pr[L_{k}^\leq] = 1-O(n^{-2})$ by \Cref{lma:length_exploit_whp}, and $\Pr[L_{k+1}^\leq] = 1-O(n^{-2})$ since the progress in a diversity phase is distributed as $1+\geo(1/2)$ and is bounded by $2\log_2 n$ with probability at least $1-1/n^2$. Hence, $\Pr[\calL] = 1-O(n^{-2})$. Since $\bar H_{k+2} \in [0,n]$,
    \begin{align}
        \EE[\bar H_{k+2}] & = \EE[\bar H_{k+2} \cond \calL]\cdot\Pr[\calL] + \EE[\bar H_{k+2} \cond \neg \calL]\cdot\Pr[\neg \calL] \label{eq:negligible_error2}\\
        & = \EE[\bar H_{k+2} \cond \calL] (1-O(n^{-2}))+o(1) 
        = \EE[\bar H_{k+2} \cond \calL] \pm o(1).\nonumber
    \end{align}
In the following we will thus compute $\EE[\bar H_{k+2} \cond \calL]$. Note that by the same calculation as~\eqref{eq:negligible_error} and~\eqref{eq:negligible_error2}, Lemmas \ref{lma:div-in-div-phase} and \ref{lma:div-in-exploit} also hold when conditioning on $\calL$, at the cost of an additive $\pm o(1)$ error term. Hence \Cref{lma:div-in-exploit} yields
    \begin{align} \label{eq:change-with-one-phase}
        \EE[\bar H_{k+1}\cond \calL] 
        & \le \EE[H_{k} \cond \calL]+ 5c_{\text{len}} \cdot \alpha_k n  \log n/\sqrt{\lambda} +o(1)\nonumber \\
        & \stackrel{\eqref{eq:induction_hypothesis}}{\le} (c+5c_{\text{len}})\alpha_k n  \log n/\sqrt{\lambda} + o(1),
    \end{align}
and by \Cref{lma:div-in-div-phase},
    \begin{align}\label{eq:change-with-two-phases}
        \EE[\bar H_{k+2}] &\le  
        \EE[\bar H_{k+2}\cond \calL] +o(1) \le (1-\tfrac{1}{4e^2})\EE[\bar H_{k+1} \cond \calL]+ o(1)\nonumber \\
        & \le (1-\tfrac{1}{4e^2})(c+5c_{\text{len}})\alpha_k n  \log n/\sqrt{\lambda} + o(1)\nonumber\\
        & \le 0.99c \cdot \alpha_k n \log n/\sqrt{\lambda},
    \end{align}
where we absorbed the $o(1)$ term into the constant in the last term. This is almost the desired inductive bound, except that we need $\alpha_{k+2}$ instead of $\alpha_k$. However, conditional on $\calL$ and for large enough~$n$,
\begin{align*}
\alpha_{k+2}n = \alpha_kn -L_k-L_{k+1} \ge \alpha_k n - 2L_{\max} \ge 0.99\alpha_k n,
\end{align*}
because $L_{\max} = O(\sqrt{n}\log n)$ and $\alpha n = \Omega(n^{2/3})$. This implies that $0.99\alpha_k \le \alpha_{k+2}$. Hence we may continue \eqref{eq:change-with-two-phases} as 
    \begin{align*}
        \EE[\bar H_{k+2}] &\le \EE[\bar H_{k+2}\cond \calL] +o(1)\le c \cdot\alpha_{k+2} n  \log n/\sqrt{\lambda},
    \end{align*}
which yields \eqref{eq:induction_hypothesis} for $k+2$. By induction, this shows the statement for all even $k$. For the remaining points in time, which are not the start of an exploitation phase, we use very similar arguments. Equation~\eqref{eq:change-with-one-phase} tells us that after the $k$-th phase, $k$ even, we have $\EE[\bar H_{k+1}] = O(\alpha_k n\log n /\sqrt{n})$. The same argument also shows that $\EE[\bar H_{t}] = O(\alpha_k n\log n /\sqrt{n})$ for all times $t$ during the $k$-th phase. For the $(k+1)$-th phase, since it is a diversity phase $\bar H$ can only decrease during this phase, so we also have $\EE[\bar H_{t}] \le \EE[\bar H_{k+1}] = O(\alpha_k n\log n /\sqrt{n})$ for all times $t$ during the $(k+1)$-th phase. As before $\alpha$ only changes by a $(1-o(1))$ factor during one phase, hence we also obtain $\EE[\bar H_{t}] = O(\alpha_t n\log n /\sqrt{n})$ for all $t$ in either phase $k$ or phase $k+1$, at the expense of a slightly larger hidden constant $c'$. Therefore, the statement~\eqref{eq:in_between} for all times $t$ follows.
\end{proof}


\section{Experiments}\label{sec:experiments}
This section is devoted to an empirical evaluation of DEGA to (i) validate our theoretical findings on \textsc{LeadingOnes}, (ii) discuss variants of the DEGA, and (iii) assess performance on classical benchmarks. We investigate the performance on \textsc{OneMax} (\textsc{OM}), \textsc{LeadingOnes} (\textsc{LO}), \emph{Linear Functions} with harmonic weights \emph{(LFHW)} and the \emph{Maximum Independent Vertex Set (MIVS)}. Also, runtimes are compared to the established \tpoga, \ollga and \UMDA~\cite{muhlenbein1997equation}. Details on the parameter settings and benchmarks follow below.
\smallskip

\noindent \textbf{Setup.} For \textsc{LO}, \textsc{OM} and \emph{LFHW}, we study the number of evaluations before finding the optimum (``runtime''). For \emph{MIVS} we have to benchmark slightly differently, as we mostly do not reach the global optimum for any algorithm. The setup for \emph{MIVS} is described in the supplement. 
We therefore focus on the setup for the other benchmarks: All experiments are based on $50$ independent runs per problem size. We generate $10$ log-spaced problem sizes in a range $[n_{\text{start}} = 100, n_{\text{end}}]$. The value of $n_{\text{end}}$ may vary for different benchmarks (since simulation on \textsc{LO} is slower than on \textsc{OM}). We denote by $\bar{T}(n)$ the mean runtime (number of fitness evaluations) and by $\tilde{T}(n)$ the median runtime; error bars or intervals indicate one standard deviation where applicable. We normalize the average runtimes $\bar T(n)$ by $n^2$ (\textsc{LO}), $n\log n$ (\textsc{OM}) and $n\log n$ (\emph{LFHW}). For reproducibility, our implementation and data are available on GitHub at \url{https://github.com/FOGA2025-DEGA/DEGA}.\smallskip

\noindent \textbf{\textsc{LeadingOnes} Asymptotics}
We begin by testing the algorithm as described in \Cref{sec:algo}. In \Cref{fig:lambda-asymptotics}, the average optimization time, divided by $n^2$, is shown for a range of $\lambda$'s. We plot for a given $\lambda_i$, $\bar T(n)$ as well as the runtime guarantee from \Cref{thm:mainresult} for that specific $\lambda_i$. Interestingly, the fit is remarkably tight for the proof-optimal $\lambda_1= (n\log n)^{2/3}$, despite our bound not being tight for a given $\lambda$. This can be explained by the structure of the runtime, which is a sum of two terms. 

Recall
$
    \EE[T] = O ( \lambda n + n^2\log n/\sqrt \lambda )
$. 
The $\lambda n$-term is asymptotically tight, so even if the $n^2\log n /\sqrt \lambda$-term is overestimated, choosing $\lambda = (n\log n)^{2/3}$ to theoretically balance the two terms will yield asymptotic time $\Theta(n^{5/3}(\log n)^{2/3})$. Therefore, we would expect that for $\lambda \ge (n\log n)^{2/3}$ our asymptotic prediction is correct, but for $\lambda = o( (n\log n)^{2/3})$ the proven bound may overestimate the actual runtime. This is indeed what we find, as for any $\lambda_i, i > 1$ we have $\lambda_i = o(\lambda_1)$ and the runtime guarantees are overestimating the empirical runtimes in the figure.
\begin{figure*}[t]  
  \centering
  \includegraphics[width=\textwidth]{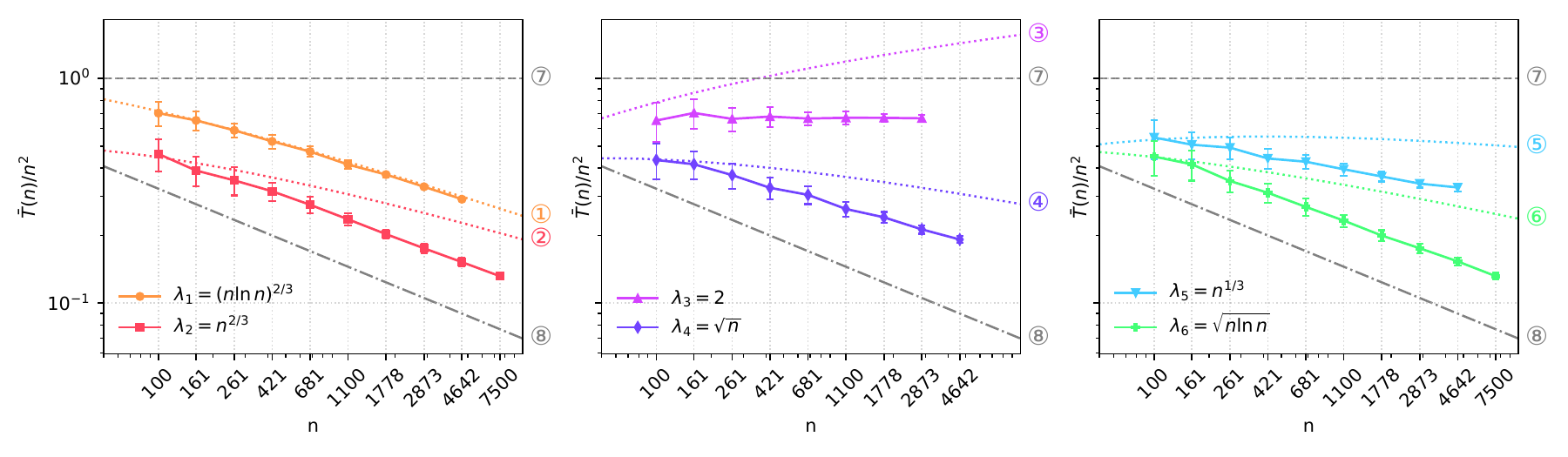}
  \caption{Log–log plot of mean runtime of the \todega for different $\lambda$’s on \leadingones. 
  Dashed lines \raisebox{.5pt}{\textcircled{\raisebox{-.7pt}{1}}}–\raisebox{.5pt}{\textcircled{\raisebox{-.7pt}{6}}}
  show the runtime guarantees from \Cref{thm:mainresult} for the respective $\lambda_i$’s.
  \raisebox{.5pt}{\textcircled{\raisebox{-.7pt}{7}}} and
  \raisebox{.5pt}{\textcircled{\raisebox{-.7pt}{8}}}
  mark the functions $n^2$ and $n^{5/3}$ for comparison. 
  Bounds are scaled by suitable constants for visibility.}
  \label{fig:lambda-asymptotics}
\end{figure*}
We also observe that several values of $\lambda$ that are smaller than the proof-optimal value, namely $\lambda_2, \lambda_4, \lambda_6$ exhibit similar performance on the log-log plot. Additionally, we ran the algorithm with $\lambda = 2$. It is clear that $\lambda = 2$ causes the algorithm to closely resemble a vanilla \tpoga. Unsurprisingly, the runtime is therefore not sub-quadratic for $\lambda = 2$. For all $\lambda_i$ (from \Cref{fig:lambda-asymptotics}), we $\log$-transformed $x_j = \log n_j$ and $y_j = \log T(n_j)$ and performed linear regression on $(X,Y) = (x_j, y_j)_j$ to approximate the polynomial degree of the runtime for a given $\lambda_i$. We skipped the first four $n$ values to reduce the influence of the standard deviation. The best fitting $a$ for $Y = aX + b$ are displayed in \Cref{tab:lambda_a_values}. Note that this does not contradict our claim of a $ O(n^{5/3}\log n) = O(n^{1.\bar{66}} \log n)$ runtime. The $\log$-term still influences the slope calculated by the regression. Considering a larger range of $n$'s will result in the $\log$-transformed data not being  linear. \smallskip

\begin{table}[ht]
\centering
\begin{tabular}{|c|c|c|c|c|c|c|c|}
\hline
& $\lambda_1$ & $\lambda_2$ & $\lambda_3$ & $\lambda_4$ & $\lambda_5$ & $\lambda_6$ \\ \hline
$ a$ & $1.749 $       & $1.695$       & $2.001$       & $1.766$       & $1.859$       & $1.706$       \\ \hline
\end{tabular}
\caption{Empirical polynomial degree of the runtime for different $\lambda$.} 
\label{tab:lambda_a_values}
\end{table}

We now turn our attention to other function classes and examine how $\dega$ performs against other algorithms. To this end, we also discuss variations of the standard \todega, which we call $A$ from now on. A version $A'$, described in \Cref{fig:A'}, includes the ideas discussed in Section~\ref{sec:intro} to make the algorithm more robust. Note that in case of mutation, the offspring can only replace its parent but never the other point, even if it is fitter. $A'$ uses an adaptive $\lambda = H(x,y)$ as bias, but still only exchanges one bit in expectation, and we suspect that an option of exchanging more bits would be beneficial on benchmarks like MIVS and \jump.




\begin{figure*}[t]
  \centering
  \begin{adjustbox}{width=\textwidth,center}
    \begin{tikzpicture}[node distance=2.5cm, scale=1, transform shape]
      \node (start) [startstop] {Start $P = \{x,\bar{x}\}$};
      \node (repeat) [process, fill=mymagenta!50, right of=start, xshift=0.5cm, align=center]
            {Until Convergence\\ $x^1,x^2 \leftarrow P$};
      
      \node (mutation) [process, right of=repeat, yshift=1cm, xshift=1cm]
            {Select $x' \in P$ u.a.r.};
      \node (mutation_description) [right of=repeat, yshift=0.25cm, xshift=1cm]
            {Let $x'' \in P \backslash \{x'\}$};
      \node (perform_mutation) [process, right of=mutation, xshift=1cm]
            {$y \leftarrow$ mutate $x'$};
      \node (replace_mutation) [process, right of=perform_mutation, xshift=2.25cm, align=center]
            {Replace $x' \leftarrow y$ if $f(y) > f(x')$ or \\ 
             $f(y) = f(x') \land H(y, x'') > H(x',x'')$};
      
      \node (co) [process, right of=repeat, yshift=-1.5cm, xshift=1.5cm]
            {$y \leftarrow$ Crossover($x^1$, $x^2$, $1/2$)};
      \node (co_description) [right of=repeat, yshift=-0.75cm, xshift=1.5cm]
            {Let $x' \in P$ with $f(x') = f_{\text{min}}$};
      
      \node (compare) [decision, right of=co, xshift=1.5cm]
            {$f(y) > f(x')$};
      \node (hamming) [right of=compare, yshift=1cm, xshift=2.5cm]
            {$h = H(x',y)$};
      \node (loop) [process, right of=compare, xshift=2.5cm, align=left]
            {For $i = 1, \dots, h \log(n)$\\
             $\quad z \leftarrow$ Crossover($x'$, $y$, $1/h$)\\
             $\quad$if $f(z) > f(x')$ then $x' \leftarrow z$ and \textbf{break}};
      
      \draw [arrow] (start.east)  -- (repeat.west);
      \draw [arrow] (repeat) |- (mutation.west) node[midway, above] {w.p. 1/2};
      \draw [arrow] (mutation.east)  -- (perform_mutation.west);
      \draw [arrow] (perform_mutation.east)  -- (replace_mutation.west);
      \draw [arrow] (replace_mutation.east)  -- ++(1,0) |- (repeat.east);
      
      \draw [arrow] (repeat) |- (co.west) node[midway, below] {w.p. 1/2};
      \draw [arrow] (co.east)  -- (compare.west);
      \draw [arrow] (compare.north) |- (repeat.east) node[near start, right] {No};
      \draw [arrow] (compare.east) -- (loop.west) node[near start, above] {Yes};
      \draw [arrow] (loop.east)  -- ++(0.5,0) |- (repeat.east);
    \end{tikzpicture}
  \end{adjustbox}
  \caption{Flowchart of the $A'$ variant of the $\todega$.}
  \label{fig:A'}
\end{figure*}
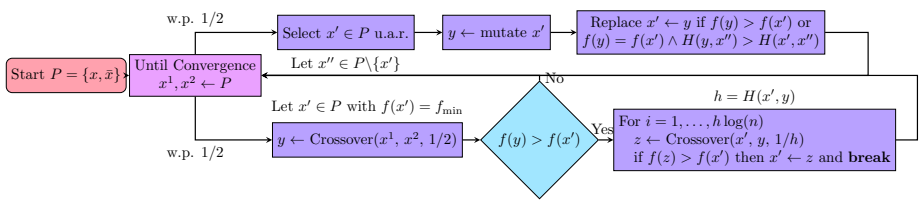

\begin{figure}[ht]
    \centering
    \includegraphics[width=\linewidth]{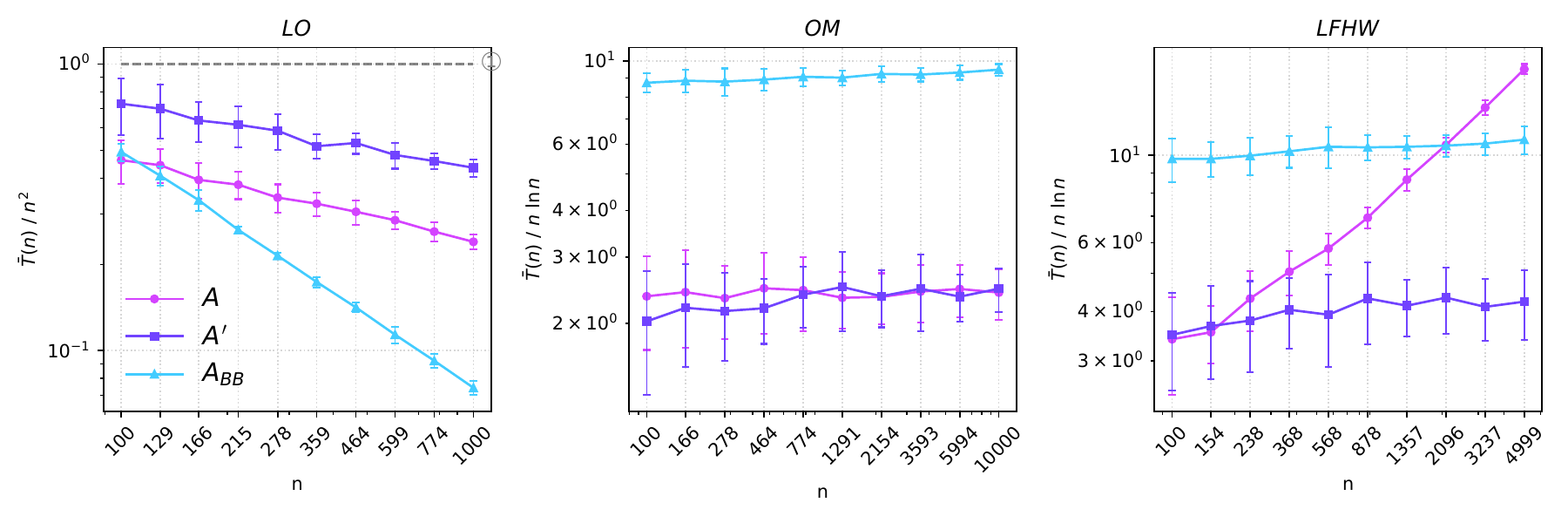}
    \caption{Log-log plot of normalized $\bar T(n)$ for \dega-variants on \leadingones (left), \onemax (middle) and \textit{linear functions} with harmonic weights (right). For $A$ we again used $\lambda = n^{2/3}$.}
    \label{fig:dega-variants-comparison}
\end{figure}

For another variant $A_{\text{BB}}$, we use an idea proposed in the context of the black-box complexity of \leadingones~\cite{doerr2011faster} in order to speed up the exploitation phases. The original $A$ requires $\EE[T_{\text{improve}}] = \Theta(\lambda)$ \LO-evaluations per improving bit.   \cite{doerr2011faster} instead uses uniform crossover between $x^1$ and $x^2$, where $x^1$ is the strictly weaker point, until a fitter candidate $y$ is found. The same procedure is then applied between $x^1$ and $y$ to further update $y$. This is iterated $O(\log n)$ times, until $x^1$ and $y$ are Hamming neighbours. Here, we obtain $A_{\text{BB}}$ from $A'$ by replacing the lower right box in~\Cref{fig:A'} by a loop of $10\log n$ uniform crossovers between $x^1$ and $y$, where the offspring replaces $y$ if it is fitter than $x^1$, and $y$ replaces $x^1$ in the end. \smallskip

\noindent \textbf{Benchmark Comparison} \label{sec:experiments:benchmarks}
In the following, unless specified otherwise, we will work with $\lambda = n^{2/3}$ for the \dega. For the \tpoga crossover is performed with $p_c = 1/2$ to obtain $y$. With $1 - p_c$ we let $y$ be a copy of a random parent. Then, $y$ is mutated to generate offspring $y'$. For the \ollga we work with $k = \lambda = \sqrt{\log n}$, which yields an expected runtime of $O(n\sqrt{\log n})$ on \onemax  ~\cite{doerr2015black}. The true asymptotic optimum involves a slightly larger $\lambda$~ \cite{doerr2021self} that gives only negligible additional speedup for the range of $n$ that we consider. 
For the \UMDA~\cite{muhlenbein1997equation} we work with $\lambda = \sqrt n \log n, \mu = \log n$. Unsurprisingly, the \dega (algorithm $A$) outperforms the other algorithms on \leadingones (\Cref{fig:benchmark-comparison}). The black-box version  $A_{\text{BB}}$ is by far the most efficient on $\leadingones$ but is not generally the best choice on other benchmarks (\Cref{fig:dega-variants-comparison,fig:benchmark-comparison}). 
\begin{figure*}[t]   
  \centering
  \includegraphics[width=\textwidth]{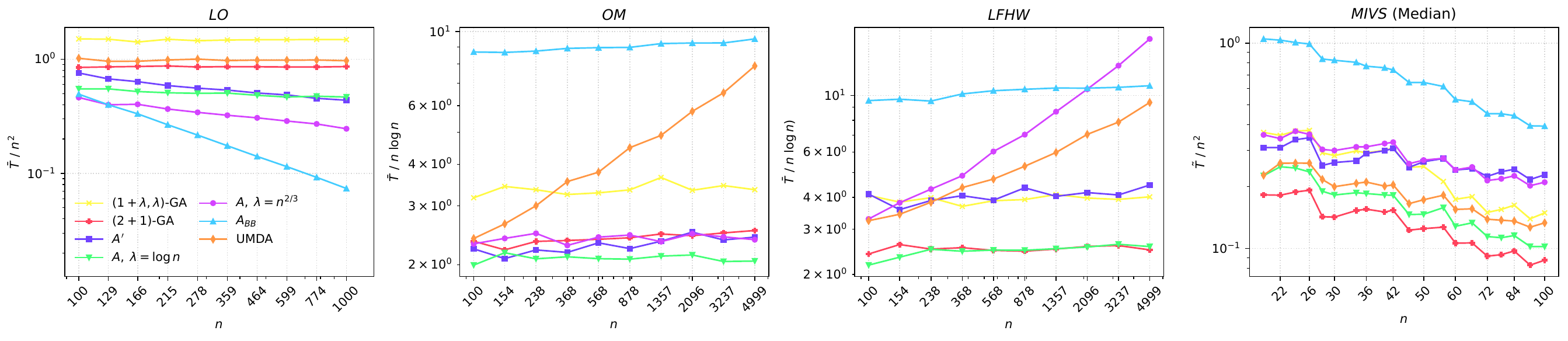}
  \caption{Log–log plot of normalized mean runtime for \leadingones\ (left),
           \onemax\ (middle-left) and linear functions with harmonic weights (middle-right). We use median runtime for MIVS (log-log, right). Mean for MIVS can be found in the supplement. We omit the standard deviation here to keep the
           graphic readable; it is reported for \dega\ in
           \Cref{fig:lambda-asymptotics,fig:dega-variants-comparison}. Recall that we limit the number of exploration phases of the \dega on MIVS, but not for $\LO$, $\OM$ or LFHW. }
  \label{fig:benchmark-comparison}
\end{figure*}
For \onemax, defined as $\onemax(x) = \sum_{i=1}^nx_i$, the algorithm performs better for smaller $\lambda$, see \Cref{fig:benchmark-comparison}. The same figure shows a similar result for \textit{linear functions with harmonic weights} (LFHW), where the fitness function is defined as $f\!: x \mapsto \sum_i i\cdot x_i$. 

From the experiments we observe that a good choice for $\lambda$ depends heavily on the function to be optimized. But we also notice that, at least for the benchmarks studied, there usually exists a $\lambda$ that provides a boost in performance. We also observe that the $A'$ adaptation is a good compromise between performance and applicability (see \Cref{fig:benchmark-comparison}), gaining an asymptotic speed-up on $\LO$ without losing much on $\OM$ and LFHW. $A'$ is outperformed by $A, \lambda = \log n$ on $\OM$ and LFHW, so studying other self-adapting strategies for $\lambda$ would be of great interest. Notice that smaller $\lambda$'s make $A$ more closely resemble a standard \tpoga. One may think that we therefore just get the \tpoga runtime for a small $\lambda$, like $\log n$. But this is not the case. The $\log n$ variant of the \dega is the fastest at optimizing $\OM$, beating the \tpoga. This may hint at some constant factor improvement. We leave the theoretical analysis of the \dega on $\OM$ to future work. As a last benchmark we consider the maximum independent vertex set (MIVS) problem. 

\subsubsection*{MIVS}

The benchmarks considered up to this point  exhibit relatively smooth fitness landscapes in which crossover can exploit small improvements. Also, we have not considered a benchmark where local maxima are difficult to escape. For $\LO$, $\OM$ and LFHW we can always find the global optimum if there is at least some mutation from time to time (mutation of one bit is enough). 
To test DEGA beyond such friendly settings, we chose the \emph{Maximum Independent Vertex Set} (MIVS) problem.  
The instance we use is the graph~\textsf{F22} from the PBO suite of \textsc{IOHprofiler}~[32]. 
For even~$n$ the graph admits a \emph{unique} optimum of size $\tfrac{n}{2}+1$.  
Because nearly every 0/1 flip upsetting a maximal independent set decreases fitness, all algorithms rapidly climb to a local optimum and then face an extremely rugged plateau.   Crossover provides no obvious shortcut here; hence MIVS is deliberately \emph{not} tailored to DEGA and serves as a stress test for its robustness. For $A$ we limit the overall number of function evaluations in a single exploitation phase to $\lambda \log n$. This is necessary to allow at least some mutation from time to time. This threshold and all other parameters were chosen ad hoc, to avoid unfair parameter tuning versus the other algorithms. \smallskip

\noindent \textbf{\emph{MIVS} testing procedure}
\label{sec:experiments:benchmarks:mivs}
Global convergence is practically infeasible for large~$n$, so we benchmark \emph{time-to-target} instead of time-to-optimum. For each $n$ we run a $(1+1)$–EA (not part of the comparison) for $3n\log n$ evaluations, repeating this $1000$~times (as the problem sizes are pretty small). The average best fitness obtained in these runs defines the target fitness $t(n)$ for the other algorithms. Each candidate algorithm is then executed $100$~times and stopped as soon as it attains $\text{round}(t(n)-1/2)$, or after a cap of $u(n)=30n\log n$ evaluations.  Unreached runs are counted with the capped time~$u(n)$. 
We report both the mean runtime (with truncation) and the median runtime, as the latter is less sensitive to the few runs that never hit the target within the budget. We use the same parameters for the \ollga, \UMDA as for the other benchmarks. \smallskip

\noindent \textbf{\emph{MIVS} results.}
Figure~\ref{fig:benchmark-MIVS} summarises the outcomes. We observe that the DEGA (Algorithm $A$) with $\lambda = \log n$ remains competitive with the $(2+1)$-GA and UMDA, but the larger $\lambda = n^{2/3}$ and the variant $A'$ perform worse, though it is unclear whether the difference is asymptotic or ``only'' a constant facor of $\approx 3$. Likely they waste too many function evaluations when there are two independent sets of different fitnesses.

\setcounter{figure}{5}
\begin{figure}[ht]
  \centering
  \includegraphics[width=\linewidth]{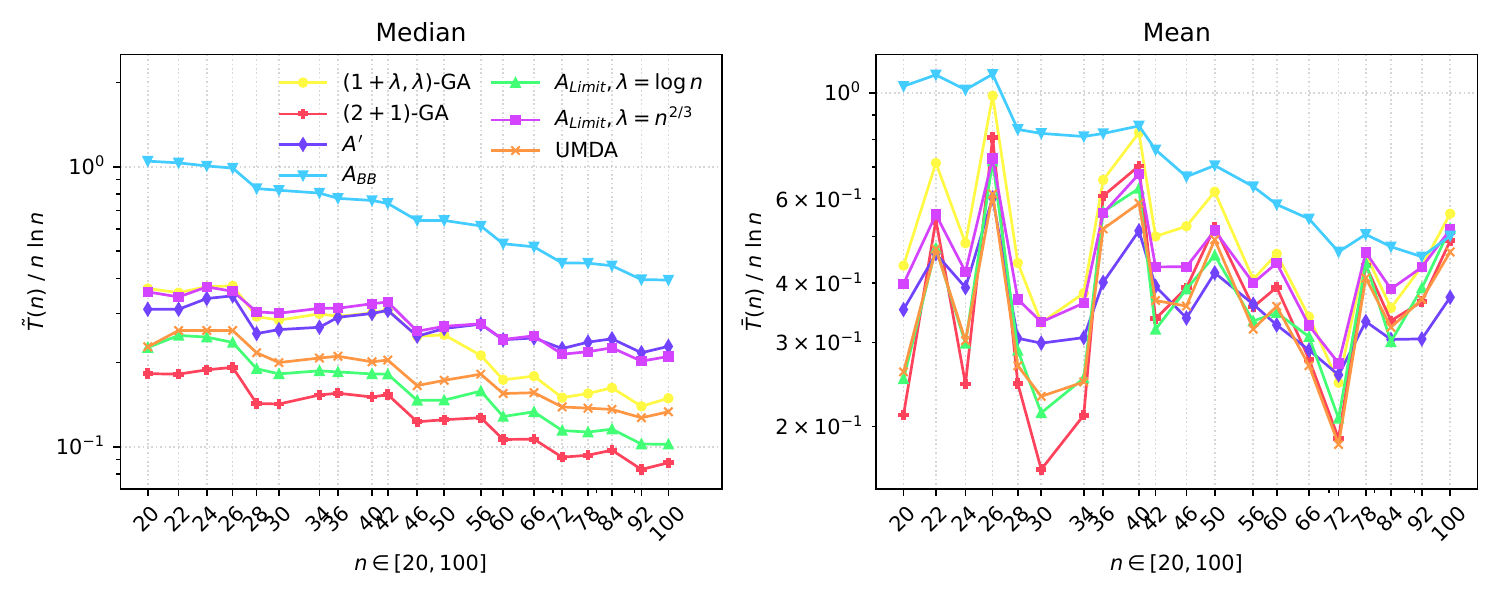}
  \caption{Median \& mean time to reach the adaptive MIVS target (log-log)}
  \label{fig:benchmark-MIVS}
\end{figure}

\subsubsection*{Conclusion on experiments.} Our experiments show that for small $\lambda$ the \dega behaves (in the worst case) very much like the well-known \tpoga. But there are problems like \onemax, where the \dega with small $\lambda$ slightly outperforms the \tpoga. There is no \dega variation that dominates across the benchmark suite. We do not claim that \dega is universally best; instead, we show that it is competitive and robust. Hence it may be a good addition to algorithm portfolios. On \leadingones it is clearly superior. 
This invites a closer look at the structural properties that make \leadingones and \onemax particularly suitable for \dega.

\section{Conclusion and Future Work}\label{sec:conclusion}
We have introduced a new paradigm of diversity-preserving exploitation of crossover (DiPEC), and given an algorithm based on this paradigm, the Diversity Exploitation Genetic Algorithm \dega. We believe that we have provided enough evidence to justify exploring the algorithm and the paradigm further in future work. Obvious next steps are to test them on a wider set of benchmarks, and to explore further the different variations of the \dega. Some important open questions include:
\begin{itemize}
    \item Do the amendments of the \dega work in practice? (E.g., adaptive choice of $\lambda$, continuous fitness functions.)
    \item How does the \dega perform for larger population sizes?
    \item Which features of a fitness landscape make the \dega faster than other algorithms, and which make it slower?
\end{itemize}
A particularly interesting question is how to integrate uniform crossover into the \dega. There may be fitness landscapes where uniform crossover offspring are preferable over biased ones, with \jump and possibly MIVS as examples. The \dega is designed to preserve diversity. Hence, the algorithm avoids on purpose accepting uniform crossover offspring. This is because every such offspring comes at a large cost for the diversity, and simply accepting them with a normal rate would defy the idea behind the \dega. For example, we believe that this would decrease performance on \leadingones strongly. However, not accepting uniform crossover offspring at all seems an extreme choice. It is open how to optimally balance the objectives of sometimes accepting uniform crossover offspring, and of keeping the damage for diversity in check. A similar trade-off is yet to be found for fitness plateaus. The \dega is conservative and accepts crossover offspring only if they give a strict fitness improvement. But this limits its ability to perform random walks on plateaus. Here, a better trade-off is desirable.

\bibliographystyle{abbrv}
\bibliography{refs}

\end{document}